\documentclass[10pt,oneside]{article}

\usepackage{lmodern}
\usepackage{times}
\usepackage{mathtools}

\usepackage{amssymb,amsmath,amsthm}
\usepackage{bbold}
\usepackage[short]{optidef}

\usepackage{framed,multirow,multicol}
\usepackage{subcaption}

\usepackage{xcolor,xspace}
\usepackage{lscape}
\usepackage{graphicx,epsfig,tikz,caption}
\usepackage[normalem]{ulem}
\usepackage{enumerate}
\usepackage{verbatim}
\usepackage{makeidx,latexsym}
\usepackage[colorlinks=true,citecolor=blue]{hyperref}%

\usepackage{bm}
\usepackage{stmaryrd}
\usepackage{lscape}

\usepackage[english]{babel}
\usepackage{bbm}
\usepackage[para]{footmisc}

\usepackage{parskip}
\usepackage[letterpaper,margin=1.1in]{geometry}

\newcommand{\jk}[1]{\textcolor{red}{#1}}
\newcommand{\mv}[1]{\textcolor{blue}{#1}}

\renewcommand{\paragraph}[1]{\noindent\textbf{#1}\quad}

\usepackage[utf8]{inputenc} 
\usepackage[T1]{fontenc}    
\usepackage{hyperref}       
\usepackage{url}            
\usepackage{booktabs}       
\usepackage{nicefrac}       
\usepackage{microtype}      
\usepackage{xcolor}         

\newcommand{\UCB}{\mathrm{UCB}}

\newcommand{\ga}{{\mathcal{G}}}
\newcommand{\ba}{{\mathcal{B}}}
\newcommand{\rot}{\varrho}

\usepackage{amsmath,bm}
\usepackage{bbm}
\usepackage{amsfonts}
\usepackage{amsthm}
\usepackage{algorithm}
\usepackage{algorithmic}

\usepackage{xcolor,xspace}
\usepackage{lscape}
\usepackage{graphicx,epsfig,tikz}
\usepackage[normalem]{ulem}
\usepackage{enumerate}
\usepackage{verbatim}
\usepackage{makeidx,latexsym}

\newtheorem{theorem}{Theorem}[section]
\newtheorem{lemma}[theorem]{Lemma}
\newtheorem{remark}[theorem]{Remark}
\makeatletter
\def\munderbar#1{\underline{\sbox\tw@{$#1$}\dp\tw@\z@\box\tw@}}
\makeatother

\title{Rotting infinitely many-armed bandits
}
\author{\normalsize
\begin{tabular}{c} Jung-hun Kim \\ KAIST \\ junghunkim@kaist.ac.kr \end{tabular} \and
\normalsize\begin{tabular}{c} Milan Vojnovi{\' c} \\ London School of Economics \\ m.vojnovic@lse.ac.uk \end{tabular} \and
\normalsize\begin{tabular}{c} Se-Young Yun \\ KAIST \\ yunseyoung@kaist.ac.kr \end{tabular} }
\date{}

\begin{document}

\maketitle

\begin{abstract} We consider the infinitely many-armed bandit problem with rotting rewards, where the mean reward of an arm decreases at each pull of the arm according to an arbitrary trend with maximum rotting rate $\varrho=o(1)$. We show that this learning problem has an $\Omega(\max\{\varrho^{1/3}T, \sqrt{T}\})$ worst-case regret lower bound where $T$ is the time horizon. We show that a matching upper bound $\tilde{O}(\max\{\varrho^{1/3}T, \sqrt{T}\})$, up to a poly-logarithmic factor, can be achieved by an algorithm that uses a UCB index for each arm and a threshold value to decide whether to continue pulling an arm or remove the arm from further consideration, when the algorithm knows the value of the maximum rotting rate $\varrho$. We also show that an $\tilde{O}(\max\{\varrho^{1/3}T, T^{3/4}\})$ regret upper bound can be achieved by an algorithm that does not know the value of $\varrho$, by using an adaptive UCB index along with an adaptive threshold value.
\end{abstract}
    
\section{Introduction}

We consider a fundamental sequential learning problem in which an agent must play one option at a time from an infinite set of options with non-stationary reward distributions, where the mean reward of an option decreases at each play of this option. This is naturally studied as the infinitely many-armed bandit problem with rotting rewards. The assumption of infinitely many arms models practical situations when there is a finite but large number of arms relative to the number of available experiments. There is an abundance of applications in which one must choose from a large set of options with rotting rewards, e.g. online advertising where arms correspond to ads and rewards decrease over exposures of an ad to a user, content recommendation systems where arms correspond to media items and rotting arises because of user boredom when watching the same content, and clinical trials where the efficacy of a medicine
may decrease because of drug tolerance when a patient takes the same medicine several times.
While there has been a lot of work on multi-armed bandits with a finite number of arms with stationary or non-stationary rewards, and an infinite number of arms with stationary rewards, not much seems to be known for the case of infinitely many arms with non-stationary rewards. 

In this paper we make first steps to understand the fundamental limits of sequential learning for infinite number of arms whose mean rewards decrease with the number of pulls---the case commonly referred to as the \emph{rested rotting bandits}. Our focus is on rotting trends where the mean reward of an arm decreases arbitrary for at most a fixed amount $\rot$ at each pull of this arm. The initial mean rewards of arms are assumed to be independent and identically distributed according to uniform distribution on $[0,1]$. The objective is to find a policy that minimizes the expected cumulative regret over a time horizon of $T$ time steps with respect to playing the best arm. 

We show that the worst-case regret for this problem is lower bounded by $\Omega(\max\{\rot^{1/3}T,\sqrt{T}\})$,
where $\rot$ is the maximum rotting rate, and show that this lower bound is tight up to a poly-logarithmic factor. 
This reveals that the rotting trend starts to have an effect on regret precisely at the threshold $\rot = \Theta(1/T^{3/2})$. Our result implies that the rotting rested bandit problem with infinitely many arms is harder than for the stationary rewards case, as in the latter case the regret lower bound is $\Omega(\sqrt{T})$ \cite{Wang}. This stands in stark contrast to the case of finite $K$ arms in which case it is known that
$\tilde{O}(\sqrt{KT})$ can be achieved for the rotting case \cite{Seznec}, which matches the optimal bound in the stationary case \cite{auer} up to a poly-logarithmic factor. 

In the case of infinitely many arms with stationary rewards, it is not possible to explore all arms to find an optimal arm, hence, it is required to find a near-optimal arm; contrast this with the case of finitely many arms, where all arms must be explored to identify an optimal arm. Further, when we consider rotting rewards, the learner must keep exploring new arms because a near-optimal arm may become suboptimal as it is being pulled. Based on this fact, we design algorithms for the rotting infinitely many-armed bandit problem to achieve tight regret bounds. We summarize our contributions in more details in what follows.  




\subsection{Summary of our contributions}

We show an $\Omega(\max\{\rot^{1/3}T,\sqrt{T}\})$ worst-case regret lower bound for the rotting rested bandit case with maximum rotting rate $\rot = o(1)$. This regret lower bound matches the regret lower bound $\Omega(\sqrt{T})$ that is known to hold for the case of stationary rewards, when rotting is sufficiently small---precisely when $\rot = O(1/T^{3/2})$. Otherwise, when $\rot = \omega(1/T^{3/2})$, the regret lower bound becomes worse than for the stationary case. 

We show that an $\tilde{O}(\max\{\rot^{1/3}T,\sqrt{T}\})$ regret can be achieved by an algorithm when the maximum rotting rate $\rot$ is known to the algorithm. This algorithm uses a UCB index to decide whether to continue pulling an arm or remove the arm from further consideration and switch to exploring a new arm by comparing the index with a threshold. This threshold is set to account for rotting of rewards. 

We further show that an $\tilde{O}(\max\{\rot^{1/3}T,T^{3/4}\})$ regret can be achieved by an algorithm that does not know the value of the maximum rotting rate $\rot$. This algorithm uses an adaptive UCB index and an adaptive threshold value to compare the UCB index of an arm with the threshold to decide whether to continue pulling this arm or remove the arm from further consideration. This upper bound matches the lower bound up to poly-logarithmic factors when the rotting rate $\rot$ is sufficiently large, i.e. when $\rot=\Omega(1/T^{3/4})$.

We present results of numerical experiments for randomly generated problem instances of rotting infinitely many-armed bandits. These results validate the insights derived from our theoretical results.


\subsection{Related work} 


The work on multi-armed bandits can be distinguished with respect to two criteria, first whether the number of arms is finite or infinite, and second whether rewards of arms are stationary or non-stationary. For the case of non-stationary rewards, we can further distinguish rested from restless multi-armed bandit problems --- in the former case, an arm's distribution of reward may change only when the arm is pulled, while in the latter case, it may change at each time step. Our work falls in the category of multi-armed bandit problems with infinitely many non-stationary rested arms. 

The case of finitely many arms with stationary rewards has been studied by many, following on \cite{lai,auer2002finite}. There exist algorithms having $\tilde{O}(\sqrt{KT})$ worst-case regret, where $K$ is the number of arms, and this matches the lower bound $\Omega(\sqrt{KT})$ up to a poly-logarithmic factor \cite{auer,Slivkins}. 

We next discuss the case of finitely many arms with non-stationary rewards. The non-stationarity in rewards can be quantified by the number of abrupt changes or a variation budget, which is referred to as \textit{abrupt-changing} and \textit{slow-varying} environments, respectively. The non-stationary environments were studied by \cite{auer,Garivier,Besbes,auer2019adaptively} in which proposed algorithms are based on a strategy of adapting current state rapidly and fading old history memory (e.g. sliding window, discount factor, and restarting). 
In addition to this, non-stationary environments were studied under various assumptions, e.g. contextual bandits and MDPs  \cite{cheung2019learning,chen2019new,zhao2020simple,russac2019weighted,cheung2020reinforcement}, mortal bandits where arms have a stochastic lifetime  \cite{Chakrabarti,Traca}, and bandits where arm rewards evolve according to a continuous-time stochastic process \cite{Brownian}.

The multi-armed bandit problem with a finite number of arms, where each arm's mean reward decays with the number of pulls of this arm, was first studied by \cite{Komiyama,Heidari,Bouneffouf,Levine}. Following \cite{Levine}, this problem is referred to as \emph{rotting bandits} problem. \cite{Levine} showed that a sliding-window algorithm has a $\tilde{O}(K^{1/3}T^{2/3})$ regret in a non-parametric rested rotting setting where the only assumption is that mean rewards are positive and non-increasing in the number of pulls. 
The non-parametric rotting bandit problem, allowing mean rewards to be negative with bounded decay, was subsequently studied by \cite{Seznec}, showing an algorithm that has 
an $\tilde{O}(\sqrt{KT})$ problem instance independent bound. \cite{Seznec2} showed that a single algorithm, an adaptive-window UCB index policy, achieves near-optimal regret for both rested and restless rotting bandits. In this paper, we follow the non-parametric rested rotting setting where mean rewards can only decrease with bounded decrements. 

We next discuss the case of infinitely many arms with stationary rewards. \cite{Berry,Bonald} proposed algorithms with asymptotically optimal regret $O(\sqrt{T})$ for the case of arms with Bernoulli rewards and independent mean values according to uniform distribution on $[0,1]$. \cite{Wang} studied the case where the mean reward distribution has support on $[0,\mu^*]$ with $\mu^*\leq 1$, and for each arm $a$ the distribution of mean reward $\mu(a)$ is such that $\mathbb{P}(\mu(a) \geq \mu^* - z) = \Theta(z^\beta)$, for some $\beta > 0$. 
\cite{Carpentier} studied the same problem but focused on simple regret, defined as the instantaneous regret at time step $T$. \cite{Bayati} showed that a subsampled UCB algorithm (SSUCB) that samples $\Theta(\sqrt{T})$ arms and executes UCB only on this subset of arms has $\tilde{O}(\sqrt{T})$ regret under 1-sub-Gaussian rewards with mean rewards according to uniform distribution on $[0,1]$. In this setting, for mean reward distributions such that there is a large enough number of near-optimal arms, an algorithm may find a near-optimal arm by exploring a restricted number of arms. There also exist several works dealing with infinitely many arms under structured reward functions such as contextual linear bandits \cite{abbasi2011improved} and Lipschitz bandits \cite{bubeck2011lipschitz}. In this paper, however, we focus on infinitely many arms under a mean reward distribution, where the structured-reward assumptions may not hold because of rotting.

Our work is different from the work discussed in this section in that we consider the case of infinitely many arms with non-stationary rotting arms. 
In the case of rotting bandits with a finite number of arms, as we mentioned, \cite{Seznec,Seznec2} achieves worst-case regret bound  $\tilde{O}(\sqrt{KT})$ which matches the near-optimal regret in the stationary stochastic setting. This result indicates that the rotting in the finitely many arms setting is not a harder problem than in the stationary rewards setting. However, in the setting of infinitely many arms, rotting of rewards makes the problem harder than in the stationary rewards case.
This is because the value of the optimal mean reward is not decreasing as the arms are being pulled as there are infinitely many near-optimal arms, which requires an additional exploration to recurrently search for a new optimal arm outside of the set of already pulled arms. 
Our algorithms are different from previously-proposed algorithms for the case of infinitely many arms with stationary rewards in keeping to explore new arms to find a near-optimal arm over time because of rotting rewards. In more details, our algorithms use UCB policies to decide whether to continue pulling an arm or remove the arm from further consideration and explore a new arm, by comparing its UCB index with a threshold which is adjusted by using the rotting rate or an estimated value of the rotting rate.



\section{Problem formulation}
\label{sec:prob}

We consider a non-stationary bandit problem with infinitely many arms where the reward distributions of arms vary over time. We consider the case when the mean reward of an arm may decrease only when this arm is pulled by an agent that uses a policy $\pi$, which is referred to as the \emph{rested rotting bandit} setting. Let $\mathcal{A}$ be an infinite set of arms, $\mu_t(a)$ be the mean reward of arm $a$ at time $t$ before pulling an arm at time $t$, and $n_t(a)$ be the number of times arm $a\in\mathcal{A}$ is pulled by $\pi$ before time $t$. Also, denote by $r_t$ the stochastic reward gained by pulling arm $a_t^\pi$ at time $t$. Let $r_t=\mu_t(a_t^\pi)+\eta_t$ where $\eta_t$ is a noise term with a $1$-sub-Gaussian distribution. We assume that initial mean rewards $\{\mu_1(a)\}_{a\in \mathcal{A}}$ are i.i.d. random variables with uniform distribution on $[0,1]$.
The rotting of arms is defined as follows. Given a rotting rate $0\le \rot_t\le \rot$ at time $t\ge1$ with maximum rotting rate $\rot=o(1)$, the mean reward of the selected arm at time $t$ changes as follows
\begin{align*}
    \mu_{t+1}(a_t^\pi)=\mu_t(a_t^\pi)-\rot_t
\end{align*}
whereas the mean rewards of other arms remain unchanged.
The mean reward of every arm $a\in\mathcal{A}$ at time $t>1$ can be represented as follows. With $0\le\rot_s\le \rot$ for all time steps $0< s < t$,  $$\mu_t(a)=\mu_1(a)-\sum_{s=1}^{t-1}\rot_s\mathbbm{1}(a_s^\pi=a).$$

 The objective is to find a policy that minimizes the expected cumulative regret over a time horizon of $T$ time steps, which for a given policy $\pi$ is defined as follows 
$$
\mathbb{E}[R^\pi(T)]=\mathbb{E}\left[\sum_{t=1}^T(1-\mu_t(a_t^\pi))\right].
$$
In the regret definition, we use that the mean reward of the optimal arm at any time $t$ is equal to $1$. This is because there is an infinite number of arms in $\mathcal{A}$ with i.i.d. mean rewards according to uniform distribution on $[0,1]$, so that there always exist sufficiently many arms whose mean rewards are close enough to $1$. In what follows, `selecting an arm' means that a policy chooses an arm in $\mathcal{A}$ before playing it and `pulling an arm' means that the policy plays a selected arm and receives a reward.

\section{Regret lower bound}
\label{sec:lb}

We first discuss two different regimes for regret depending on the value of the maximum rotting rate $\rot$. When $\rot\le 1/T^{3/2}$, the mean reward of any arm over the time horizon of $T$ time steps changes for at most $\rot T\le1/\sqrt{T}.$ Therefore, for any arm, there can be at most a gap of $1/\sqrt{T}$ between the initial mean reward and the mean reward after $T$ time steps, which causes an additional regret of at most $\sqrt{T}$ over the horizon of $T$ time steps to the case of stationary arms (i.e. when $\rot = 0$). In Theorem 3 in \cite{Wang}, the optimal regret for the stationary case, with uniform distribution of mean rewards of arms, is shown to be of the order $\sqrt{T}$. Therefore, the extra regret of $\sqrt{T}$ from the rotting with $\rot\le 1/T^{3/2}$ does not affect the order of the regret. When $\rot>1/T^{3/2}$, we expect that the regret lower bound may be different than for the stationary case. 

By analyzing the regret lower bound for the specific case with $\rot_t=\rot$ for all time steps $t>0$, we provide a lower bound for the worst-case regret with respect to arbitrary rotting as given in the following theorem.

\begin{theorem} For the rotting infinitely many-armed bandit problem, there exist rotting rates $0\le\rot_t\le \rot=o(1)$ for all time steps $t>0$ such that any policy $\pi$ has the regret over a time horizon of $T$ time steps such that
$$
\mathbb{E}[R^\pi(T)]=\Omega(\max\{\rot^{1/3}T,\sqrt{T}\}).
$$\label{thm:lower_bd_e}
\end{theorem}

From the result of the theorem, when the rotting is small enough, i.e. precisely when $\rot \leq 1/T^{3/2}$, the lower bound corresponds to  $\Omega(\sqrt{T})$, which is known to hold when rewards are  stationary. Otherwise, when the rotting is sufficiently large, then the lower bound is $\Omega(\rot^{1/3}T)$. For example, when $\rot = 1/T^\gamma$ for some $\gamma > 0$, we have the lower bound $\Omega(\sqrt{T})$, if $\gamma \geq 3/2$ (small rotting case), and,  otherwise (large rotting case), we have $\Omega(T^{1-\gamma/3})$. We can observe that $\rot=\Theta(1/T^{3/2})$ is a transition point at which the lower bound switches from $\Omega(\sqrt{T})$ to $\Omega(\rot^{1/3}T)$.



\begin{proof}[Proof sketch]
Here we present a proof sketch of the theorem with the full version of the proof provided in Appendix~\ref{app:lower_bd_e}.

We assume that $\rot_t=\rot=o(1)$ for all time steps $t>0$.
When $\rot=O( 1/T^{3/2})$, the lower bound for the stationary case of the order $\sqrt{T}$ \cite{Wang} is tight enough for the non-stationary case. This is because we only need to pay an extra regret of at most of order $\sqrt{T}$ for small $\rot$. Therefore, when $\rot=O( 1/T^{3/2})$, we have
\begin{align}
    \mathbb{E}[R^\pi(T)]=\Omega(\sqrt{T}).
    \label{eq:lowbd_small}
\end{align}
We note that even though the mean rewards are rotting in our setting, we can easily obtain \eqref{eq:lowbd_small} by following the same proof steps of Theorem 3 in \cite{Wang}. For the sake of completeness, we provide a proof in the Appendix~\ref{app:lower_bd_e}.

When $\rot=\omega(1/T^{3/2})$, however, the lower bound of the stationary case is not tight enough. Here we provide the proof of the lower bound $\Omega(\rot^{1/3}T)$ for the case when $\rot=\omega(1/T^{3/2})$.
For showing the lower bound, we will classify each arm to be either bad or good or else according to the definition given shortly. To distinguish bad and good arms, we use two thresholds $1-c$ and $1-\delta$, respectively, where $c$ and $\delta$ are such that $0< 1-c < 1-\delta < 1$, $\delta=\rot^{1/3}$, and $c$ is a constant. An arm $a$ is said to be a \emph{bad arm} if $\mu_1(a) \leq 1-c$, and is said to be a \emph{good arm} if $\mu_1(a) > 1-\delta$.


Let $N_T$ be the number of distinct selected good arms until time step $T$.
We separately consider two cases when $N_T < m$ and $N_T\ge m$, where  $m=\lceil (1/2)T\rot^{2/3}\rceil$, and show that each case has $\Omega(\rot^{1/3}T)$ as the regret lower bound. The main ideas for each case are outlined as follows. When the number of selected good arms is relatively small ($N_T<m$), any policy $\pi$ must pull arms with mean rewards less than $1-\delta$ at least $T/2$ time steps until $T$, amounting to at least $\delta$ regret for each pull (gap between 1 and mean reward of a pulled arm). Therefore, the regret is lower bounded by $\Omega(\delta(T/2))=\Omega(\rot^{1/3}T)$. When the number of selected good arms is relatively large ($N_T\ge m$), we can show that any policy $\pi$ is likely to select at least of order $\rot^{1/3}T$ number of distinct bad arms until $T$. From the fact that the selected bad arms are pulled at least once and each pull adds a constant regret of value at least $c$, the regret is shown to be lower bounded by  $\Omega (c\rot^{1/3}T)= \Omega (\rot^{1/3}T)$.
Therefore, when $\rot=\omega(1/T^{3/2})$, we can obtain
\begin{align}
  &\mathbb{E}[R^\pi(T)]\cr&=\mathbb{E}[R^\pi(T)\mathbbm{1}(N_T<m)+R^\pi(T)\mathbb{1}(N_T\ge m)]\cr&=\Omega(\rot^{1/3}T).    \label{eq:lowbd_large}
  \end{align}
Finally, from \eqref{eq:lowbd_small} and \eqref{eq:lowbd_large}, we have 
$\mathbb{E}[R^\pi(T)]=\Omega(\max\{\rot^{1/3}T,\sqrt{T}\}).$
\end{proof}

\section{Algorithms and regret upper bounds}
\label{sec:algo}

In this section, we first present an algorithm for the rested rotting bandit problem with infinitely many arms for the case when the algorithm knows the value of the maximum rotting rate. We show a regret upper bound of the algorithm that matches the regret lower bound in Theorem~\ref{thm:lower_bd_e} up to a poly-logarithmic factor. Second, we present an algorithm that does not know the maximum rotting rate and show a regret upper bound that matches the regret lower bound up to a poly-logarithmic factor, when the maximum rotting rate is large enough.

\begin{algorithm}[t]
\caption{UCB-Threshold Policy (UCB-TP)}
\begin{algorithmic}
\STATE Given: $T, \delta, \mathcal{A}$; Initialize: $ \mathcal{A}^\prime\leftarrow\mathcal{A}$
\STATE Select an arm $a\in\mathcal{A}^\prime$
\STATE Pull arm $a$ and get reward $r_1$
\FOR{$t=2,\dots,T$}
\STATE Update the initial mean reward estimator $\tilde{\mu}_t^o(a)$
\IF{$\UCB_t(a)\ge 1-\delta$}
\STATE Pull arm $a$ and get reward $r_t$
\ELSE
\STATE$\mathcal{A}^\prime\leftarrow \mathcal{A}^\prime/\{a\}$
\STATE Select an arm $a\in\mathcal{A}^\prime$
\STATE Pull arm $a$ and get reward $r_t$
\ENDIF 
\ENDFOR
\end{algorithmic}
\label{alg:Alg1}
\end{algorithm}

\subsection{An algorithm knowing maximum rotting rate}

We present an algorithm which requires knowledge of the maximum rotting rate in Algorithm~\ref{alg:Alg1}. The algorithm selects an arm and pulls this arm as long as the arm is tested to be a good arm, by using a test comparing an upper confidence bound of this arm with a threshold value. Specifically, if $a$ is the selected arm at time step $t$, the algorithm computes an estimator $\tilde{\mu}_t^o(a)$ of the initial mean reward of arm $a$ and uses this estimator to compute an estimator of the mean reward of arm $a$ at time step $t$, considering the worst-case rotting rate $\rot$ for the estimators. Comparing the upper confidence bound for the mean reward with the threshold $1-\delta$, the algorithm tests whether the arm is a good arm. If the arm is tested to be a good arm, then the algorithm continues to pull this arm. Otherwise, it discards the arm and selects a new one, and repeats the procedure described above until time horizon $T$. 
 We consider Algorithm~\ref{alg:Alg1} with the initial mean reward estimator defined as
$$
\tilde{\mu}_t^o(a)\coloneqq\frac{\sum_{s=1}^{t-1}(r_s+\rot n_{s}(a))\mathbbm{1}(a_s=a)}{n_t(a)},
$$ 
where $n_t(a)=\sum_{s=1}^{t-1} \mathbbm{1}(a_s=a)$ and the upper confidence bound term defined as 
$$
\UCB_t(a)\coloneqq \tilde{\mu}_t^o(a)-\rot n_t(a)+\sqrt{10\log(T)/n_t(a)}.
$$
Note that  when $\rot_s=\rot$ for all $0<s<t$, $\tilde{\mu}_t^o(a)$ is an unbiased estimator of the initial mean reward $\mu_1(a)$ of arm $a$ 
and $\tilde{\mu}_t^o(a)-\rot n_t(a)$ is an unbiased estimator of the mean reward $\mu_t(a)$ of arm $a$ at time step $t$. The upper confidence bound $\UCB_t(a)$ follows the standard definition of an upper confidence bound. By designing the estimators to deal with the maximum rotting rate $\rot$, for  any arbitrary $\rho_t\le \rot$ for all time steps $t>0$, we show that it has a near-optimal worst-case regret upper bound in the following theorem.

\begin{theorem} For the policy $\pi$ defined by Algorithm~\ref{alg:Alg1} with $\delta=\max\{\rot^{1/3},1/\sqrt{T}\}$, and $\rot=o(1)$, the regret satisfies
$$
\mathbb{E}[R^\pi(T)]=\tilde{O}(\max\{\rot^{1/3}T,\sqrt{T}\}).
$$
\label{thm:R_upper_bd_e}
\end{theorem}
\begin{proof}[Proof sketch]
Here we present a proof sketch of the theorem with a full version of the proof provided in Appendix~\ref{sec:proof_thm_R_upper_bd_e}.

 Observe that initial mean rewards of selected arms are i.i.d. random variables with uniform distribution on $[0,1]$. We first define arms to be good or bad arms depending on the initial mean rewards. We assume $\rot=o(1)$ and set $\delta=\max\{\rot^{1/3},1/\sqrt{T}\}$. Then we define an arm $a$ to be a good arm if $\Delta(a)\le \delta/2$, where $\Delta(a)=1-\mu_1(a)$, and otherwise, $a$ is a bad arm. Because of rotting, initially good arm may become bad by pulling the arm several times. Therefore, the policy $\pi$ may select several good arms over the entire time and we analyze the regret over time episodes defined by the selections of good arms. Given the policy, we refer to the period starting from selecting the $i-1$-st good arm until selecting the $i$-th good arm as the $i$-th episode. Because of the uniform distribution of mean rewards with small $\delta/2$, it is likely to have several consecutive selected bad arms in each episode.
 
\begin{figure}[t]
\includegraphics[width=1\linewidth]{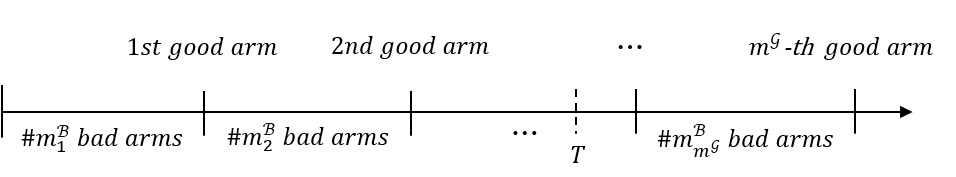}
\caption{Episodes in a time line.}
\label{fig:episode}
\end{figure}
 
We do an episodic analysis. We first analyze the expected regret per episode and multiply it by the expected number of episodes in $T$. However, this proof strategy has an issue that the regret of each episode and the number of episodes in $T$ are not independent. To resolve the issue, we fix the number of episodes to $m^\ga$ and analyze the regret not for $T$ but for $m^\ga$ episodes. Note that $m^\ga$ is a fixed value, and the time after $m^\ga$ episodes can exceed $T$. For obtaining a regret bound, we set $m^{\ga}$ so that the total number of time steps for $m^\ga$ episodes is larger than $T$ with high probability and thus $R^\pi(T)\le R^{\pi}_{m^\ga}$  where $R^{\pi}_{m^\ga}$ is the regret accumulated for $m^\ga$ episodes. For the regret analysis, we denote by $m^{\ba}_i$ the number of selections of distinct bad arms in the $i$-th episode. See Figure~\ref{fig:episode} for an illustration of the episodes in a time line.

In what follows, we provide an overview of the regret analysis by considering two separate cases depending on the value of the maximum rotting rate $\rot$; one for a large rotting case, and the other for a small rotting case, which we may interpret as a near-stationary case. 
For the analysis, we use $R_i^\ga$ to denote the regret accumulated by pulling the good arm in the $i$-th episode, and $R_{i,j}^\ba$ to denote the regret accumulated by pulling the $j$-th bad arm in the $i$-th episode.

  \textbf{Case of large rotting $\rot=\omega(1/T^{3/2})$}: We first show that by setting $m^\ga=\lceil2T\rot^{2/3}\rceil$, we have $R^\pi(T)\le R_{m^\ga}^\pi$. If the policy selects a good arm $a$, where $\mu_1(a)\ge 1- \delta/2$, then it must pull the arm at least $\delta/(2\rot)$ times, with high probability, to decrease the mean reward below the threshold $1-\delta$. Then the total number of time steps for $m^\ga$ episodes is at least $(\delta/(2\rot))m^\ga=(1/(2\rot^{2/3}))\lceil2T\rot^{2/3}\rceil\ge T$. This implies that $R^\pi(T)\le R_{m^\ga}^\pi$. We next provide a bound for $\mathbb{E}[R^\pi(T)]$ using $\mathbb{E}[R^\pi_{m^\ga}]$. For bounding $\mathbb{E}[R^{\pi}_{m^\ga}]$, 
  we can show that for any $i\in[m^\ga]$ and $j\in[m^\ba_i]$, we have
\begin{align}
    \mathbb{E}[R_i^\ga]
    =\tilde{O}\left(\frac{1}{\rot^{1/3}}\right)  \text{ and }
    \mathbb{E}[R_{i,j}^\ba]= \tilde{O}\left(1\right).\label{eq:R_good_bad_bd}
\end{align}
 Observe that $m^\ba_1,\ldots, m^\ba_{m^\ga}$ are i.i.d. random variables with geometric distribution with parameter $\delta/2$. Therefore, for any non-negative integer $k$, we have $\mathbb{P}(m^\ba_i=k)=(1-\delta/2)^k\delta/2$ and $\mathbb{E}[m_i^\ba]=(2/\delta)-1$ for all $i\in [m^\ga]$. We have set $\delta=\max\{\rot^{1/3},1/\sqrt{T}\}$.
  Then when $\rot=\omega(1/T^{3/2})$, with $m^\ga=\lceil 2T\rot^{2/3}\rceil$ and from \eqref{eq:R_good_bad_bd}, $\mathbb{E}[m_i^\ba]=2/\delta-1$, and $\delta=\rot^{1/3}$, we have
\begin{align}
 &\mathbb{E}[R^\pi(T)]=O(\mathbb{E}[R^\pi_{m^\ga}])\cr &=O\left(\mathbb{E}\left[\sum_{i\in[m^\ga]}\left(R^\ga_i+\sum_{j\in[m^\ba_i]}R^\ba_{i,j}\right)\right]\right)\cr
&= \tilde{O}\left(\left(\frac{1}{\rot^{1/3}}+\mathbb{E}[m_i^\ba]\right)m^\ga\right)=\tilde{O}\left(\rot^{1/3}T\right).
 \label{eq:regret_bd_large_ps}
\end{align}
 
 \textbf{Case of small rotting $\rot=O( 1/T^{3/2})$}: By setting $m^\ga=C$ for some large constant $C>0$, we can show that $R^\pi(T)\le R^\pi_{m^\ga}$. This is because if the policy selects a good arm $a$, then it must pull the arm for at least order $T$ times with high probability. This is because the small rotting case is a near-stationary setting so that the policy can pull a good arm for a large amount of time steps. For bounding $\mathbb{E}[R^{\pi}_{m^\ga}]$, we can show that for any $i\in[m^\ga]$ and $j\in[m^\ba_i]$, 
\begin{align}
    \mathbb{E}[R_i^\ga]
    =O(\sqrt{T})
\text{ and }    \mathbb{E}[R_{i,j}^\ba]=\tilde{O}\left(1\right).\label{eq:R_good_bad_bd_small_e}
\end{align}
  
Then when $\rot=O( 1/T^{3/2})$, with $m^\ga=C$ and from \eqref{eq:R_good_bad_bd_small_e}, $\mathbb{E}[m_i^\ba]=2/\delta-1$, and $\delta=\Theta(1/\sqrt{T})$, we have
\begin{align}
 &\mathbb{E}[R^\pi(T)]=O(\mathbb{E}[R^\pi_{m^\ga}]) \cr
&=\tilde{O}\left(\left(\sqrt{T}+\mathbb{E}[m_i^\ba]\right)m^\ga\right)=\tilde{O}(\sqrt{T}).
\label{eq:regret_bd_small_ps}
\end{align}
We note that in the small rotting case, the policy achieves $\tilde{O}(\sqrt{T})$ which matches a near-optimal bound for the stationary setting \cite{Wang}.

 \textbf{Putting the pieces together}: From \eqref{eq:regret_bd_large_ps} and \eqref{eq:regret_bd_small_ps},
by taking $\rot=o(1)$ and $\delta=\max\{\rot^{1/3},1/\sqrt{T}\}$, it follows
$$
\mathbf{E}[R^\pi(T)]=\tilde{O}(\max\{\rot^{1/3}T,\sqrt{T}\}).
$$

\end{proof}
\subsection{An algorithm not knowing maximum rotting rate}

\begin{algorithm}[t]
\caption{Adaptive UCB-Threshold Policy (AUCB-TP)}
\begin{algorithmic}
\STATE Given: $T,H,\mathcal{B}, \mathcal{A},\alpha$ 
\STATE Initialize: $\mathcal{A}^\prime\leftarrow\mathcal{A},w(\beta)\leftarrow 1$ for $ \beta\in \mathcal{B}$ 
\FOR{$i=1,2,\dots,\lceil T/H\rceil$}
\STATE Select an arm $a\in\mathcal{A}^\prime$
\STATE Pull arm $a$ and get reward $r_{(i-1)H+1}$
\STATE $p(\beta)\leftarrow(1-\alpha)\frac{w(\beta)}{\sum_{k\in \mathcal{B}}w(k)}+\alpha\frac{1}{B}$ for $\beta\in \mathcal{B}$
\STATE Select $\tilde{\beta}\leftarrow \beta$ with probability $p(\beta)$ for $\beta\in \mathcal{B}$ 
\STATE $\delta\leftarrow \tilde{\beta}^{1/3}$
\FOR{$t=(i-1)H+2,\dots,i\cdot H\wedge T$}
\IF{$\UCB_{i,t}(a,\tilde{\beta})\ge 1-\delta$} 
\STATE Pull arm $a$ and get reward $r_t$
\ELSE
\STATE$\mathcal{A}^\prime\leftarrow \mathcal{A}^\prime/\{a\}$
\STATE Select an arm $a\in\mathcal{A}^\prime$
\STATE Pull arm $a$ and get reward $r_t$
\ENDIF 
\ENDFOR
\STATE $w(\tilde{\beta})$
\STATE $\quad \leftarrow w(\tilde{\beta})\exp\left(\frac{\alpha}{Bp(\tilde{\beta})}\left(\frac{1}{2}+\frac{\sum_{t=(i-1)H}^{i\cdot H\wedge T}r_t}{2CH\log T+4\sqrt{H\log T}}\right)\right)$
\ENDFOR
\end{algorithmic}
\label{alg:Alg2}
\end{algorithm}

In this section we present an algorithm which does not require information about the maximum rotting rate $\rot$ defined in Algorithm~\ref{alg:Alg2}, and provide a regret upper bound for this algorithm. The algorithm adopts the strategy of hierarchical bandit algorithms similar to BOB (bandit-over-bandit)~\cite{cheung2019learning}. It consists of a master algorithm and several base algorithms, where EXP3 \cite{auer} is used for the master algorithm whose goal is to find a near-optimal base algorithm, and for each base algorithm, a UCB index policy similar to that in Algorithm~\ref{alg:Alg1} is used with a candidate rotting rate and an adaptive threshold to decide whether to continue pulling currently selected arm. In Algorithm~\ref{alg:Alg2}, the time horizon of $T$ time steps is partitioned into blocks of $H$ time steps. Before starting each block, the master algorithm selects a rotting estimator $\tilde{\beta}$ of the unknown maximum rotting rate $\rot$ from a set of candidate values denoted by $\mathcal{B}$. Then it runs a base algorithm over $H$ time steps which decides whether to continue pulling the selected arm based on a UCB index and a threshold tuned using the selected $\tilde{\beta}$.
By utilizing the obtained rewards over the block as a feedback for the decision of the master algorithm, it updates the master to find a near-optimal base and repeats the procedure described above until time horizon $T$. 

We note that the term $1/2+\sum_{t=(i-1)H}^{i\cdot H\wedge T}r_t/(2CH\log T+4\sqrt{H\log T})$ for some large enough constant $C>0$ in updating $w(\tilde{\beta})$ in Algorithm~\ref{alg:Alg2} is for re-scaling and translating rewards, which makes the rewards lie in $[0,1]$ with high probability. Also by optimizing the block size $H$, we can control regrets induced from the master and a base. By increasing $H$, the regret induced from the master increases and the regret induced from a base decreases. Those facts are shown later in the poof of Theorem~\ref{thm:R_upper_bd_no_e}. 


In what follows, we define the inputs $\mathcal{B}$ and $\alpha$, and the upper confidence bound index $\mathrm{UCB}_{i,t}(a,\beta)$ for $\beta\in\mathcal{B}$. 
$\mathcal{B}$ contains candidate values of $\beta$ to optimize $\mathrm{UCB}_{i,t}(a,\beta)$ and the threshold parameter $\delta$. We find that the optimal base parameter $\beta^\dagger\in\mathcal{B}$ is when $\beta^\dagger=\max\{1/H^{3/2},\rot\}$ including a clipped domain for the optimal threshold value $\delta$ as in Theorem~\ref{thm:R_upper_bd_e}. This implies that the optimized $\beta^\dagger\ge 1/H^{3/2}=2^{-3/2\log_2 H}$. Also from $\rot=o(1)$, $\beta^\dagger \leq 1/8$.
Therefore, we set \[\mathcal{B}=\{2^{-3},2^{-4},\dots,2^{-\lceil(3/2)\log_2H\rceil} \},\] in which the cardinality of the set is restricted by $O(\log H)$ which does not hurt the regret from EXP3 up to a logarithmic factor. Let $B=|\mathcal{B}|$. Then we set $\alpha=\min\{1,\sqrt{B\log B/((e-1)\lceil T/H\rceil)}\},$ which is used to guarantee a least selection probability for each base.
Let $n_{i,t}(a)$ be the number of times that arm $a\in\mathcal{A}$ is pulled by the algorithm from time step $(i-1)H+1$ before time step $t$.
Let $\tilde{\mu}_{i,t}^o(a,\beta)$ be defined as
$$
\tilde{\mu}_{i,t}^o(a,\beta)\coloneqq\frac{\sum_{s=(i-1)H+1}^{t-1}(r_s+\beta n_{i,s}(a))\mathbbm{1}(a_s=a)}{n_{i,t}(a)}
$$ 
and the $\UCB_{i,t}(a,\beta)$ index be defined as 
$$
\UCB_{i,t}(a,\beta)\coloneqq \tilde{\mu}_{i,t}^o(a,\beta)-\beta n_{i,t}(a)+\sqrt{10\log(H)/n_{i,t}(a)}.
$$
We provide a worst-case regret upper bound for Algorithm~\ref{alg:Alg2} in the following theorem.
\begin{theorem}
 With $\rot=o(1)$, for the policy $\pi$ defined by Algorithm~\ref{alg:Alg2} with $H=\lceil T^{1/2}\rceil$, the regret satisfies
 \begin{align*}
     \mathbb{E}[R^\pi(T)]=\tilde{O}(\max\{\rot^{1/3}T,T^{3/4}\}).
 \end{align*}
 \label{thm:R_upper_bd_no_e} 
 \end{theorem}

 \begin{proof}[Proof sketch]
 Here we present a proof sketch of the theorem with a full version of the proof given in Appendix~\ref{sec:R_upper_bd_no_e_proof}.
 
The policy $\pi$ consists of two strategies: EXP3 for the master and UCB-Threshold Policy (Algorithm~\ref{alg:Alg1}) for bases. We can decompose the regret into two parts: regret incurred by playing a base with $\beta\in\mathcal{B}$ over each block of $H$ time steps and regret incurred due to the master trying to find a near-optimal base parameter in $\mathcal{B}$. In what follows, we define the regret decomposition formally.
 Let $\pi_i(\beta)$ for $\beta \in \mathcal{B}$ denote the base policy with $\beta$ for time steps between $(i-1)H+1$ and $i\cdot H\wedge T$. Denote by $a_t^{\pi_i(\beta)}$ the pulled arm at time step $t$ by policy $\pi_i(\beta).$ Then, for $\beta^\dagger \in \mathcal{B}$, which is set later for a near-optimal base, we have
\begin{align}
\mathbb{E}[R^\pi(T)]&=\mathbb{E}\left[\sum_{t=1}^T 1-\sum_{i=1}^{\lceil T/H\rceil}\sum_{t=(i-1)H+1}^{i\cdot H\wedge T}\mu_t(a_t^{\pi})\right] \cr &= \mathbb{E}[R_1^\pi(T)]+\mathbb{E}[R_2^\pi(T)],\label{eq:R_decom}
\end{align}
where 
$$
R_1^\pi(T) = \sum_{t=1}^T 1-\sum_{i=1}^{\lceil T/H\rceil}\sum_{t=(i-1)H+1}^{i\cdot H\wedge T}\mu_t(a_t^{\pi_i(\beta^\dagger)}),
$$
$$
R_2^\pi(T) = \sum_{i=1}^{\lceil T/H\rceil}\sum_{t=(i-1)H+1}^{i\cdot H\wedge T}\left(\mu_t(a_t^{\pi_i(\beta^\dagger)})-\mu_t(a_t^{\pi})\right).
$$
Note that $R_1^\pi(T)$ accounts for the regret caused by the near-optimal base algorithm $\pi_i(\beta^\dagger)$ against the optimal mean reward and $R_2^\pi(T)$ accounts for the regret caused by the master algorithm by selecting a base with $\beta\in\mathcal{B}$ at every block against the base with $\beta^\dagger$ as illustrated in Figure~\ref{fig:regret_decom}. We set $\beta^\dagger$ to be the smallest value in $\mathcal{B}$ which is larger than $\max\{\rot,1/H^{3/2}\}.$ Then the base has the threshold parameter ${\beta^{\dagger}}^{1/3}$ of order $\max\{\rot^{1/3},1/\sqrt{H}\}$ which coincides with the optimal threshold parameter of Algorithm~\ref{alg:Alg1} by replacing $T$ with $H$.
We note that the policy $\pi$ does not require knowing $\beta^\dagger$ and it is defined only for the proof. 

\begin{figure}[t]
\includegraphics[width=1\linewidth]{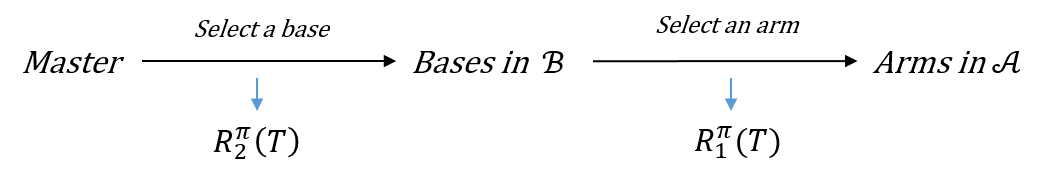}
\caption{Regret decomposition for Algorithm~\ref{alg:Alg2}.}
\label{fig:regret_decom}
\end{figure}

In what follows, we provide upper bounds for each regret component. We first provide an upper bound for $\mathbb{E}[R_1^\pi(T)]$ by following the proof steps in Theorem~\ref{thm:R_upper_bd_e}. We can easily find that regret of the base with $\beta^\dagger$ for each block of size $H$ that has the same regret bound as in Theorem~\ref{thm:R_upper_bd_e} by replacing $T$ with $H$ amounting to $\tilde{O}(\max\{\rot^{1/3}H,\sqrt{H}\})$. Then by adding the regret for $\lceil T/H\rceil$ number of blocks, we have
\begin{align}
\mathbb{E}[R_1^\pi(T)]&=\tilde{O}((T/H)\max\{\rot^{1/3}H,\sqrt{H}\})\cr &=\tilde{O}(\max\{\rot^{1/3}T,T/\sqrt{H}\}).\label{eq:R_1_bd}
\end{align}
Then we provide an upper bound for $\mathbb{E}[R_2^\pi(T)]$ using a regret bound for EXP3 in \cite{auer}. The EXP3 in policy $\pi$ selects a base in $\mathcal{B}$ before starting a block and gets feedback at the end of the block and repeats this over $\lceil T/H \rceil$ blocks. Therefore,
EXP3 in $\pi$ can be thought to be run for $\lceil T/H \rceil$ decision rounds and the number of decision options for each round is $B$. Let $Q$ be an upper bound for the absolute sum of rewards for any block with length $H$ with high probability. Then from Corollary 3.2 in \cite{auer}, we can show that
\begin{align}
\mathbb{E}[R_2^\pi(T)]=\tilde{O}(Q\sqrt{B(T/H)}).\label{eq:R_2_bd_1}
\end{align}
By considering that mean rewards may become negative because of rotting and using a Chernoff's bound, we show that with high probability
\begin{align}
    \left|\sum_{t=(i-1)H+1}^{i\cdot H\wedge T}r_t\right|\le CH\log T + 2\sqrt{H\log T}.\label{eq:Q_bd}
\end{align}
Then with $B=O(\log T)$, from \eqref{eq:R_2_bd_1} and \eqref{eq:Q_bd}, we have
\begin{align}  
\mathbb{E}[R_2^\pi(T)] &= \tilde{O}\left(H\log( T)\sqrt{B(T/H)}\right)\cr&=\tilde{O}\left(\sqrt{HT}\right).\label{eq:R_2_bd_2}
\end{align}

Finally, from \eqref{eq:R_decom}, \eqref{eq:R_1_bd}, and \eqref{eq:R_2_bd_2}, with $H=\lceil T^{1/2}\rceil$, we have
\begin{align*}
\mathbb{E}[R^\pi(T)]
&=\tilde{O}(\max\{\rot^{1/3}T,T/\sqrt{H}\}+\sqrt{HT})\cr
&=\tilde{O}(\max\{\rot^{1/3}T,T^{3/4}\}). 
\end{align*}
This concludes the proof.
\end{proof}


The regret bound for Algorithm~\ref{alg:Alg2} in Theorem~\ref{thm:R_upper_bd_no_e} is larger than or equal to that for Algorithm~\ref{alg:Alg1} in Theorem~\ref{thm:R_upper_bd_e}. This is because the master in Algorithm~\ref{alg:Alg2} needs to learn the unknown maximum rotting rate to find a near-optimal base algorithm, which produces extra regret. In the following remarks, we discuss the region of the maximum rotting rate $\rot$ for which Algorithm~\ref{alg:Alg2} achieves the near-optimal regret and discuss the computation and memory efficiency of our algorithms.

\begin{remark}
When $\rot=\Omega(1/T^{3/4})$, Algorithm~\ref{alg:Alg2} achieves the optimal regret bound $\tilde{O}(\rot^{1/3}T)$ up to a poly-logarithmic factor. This is because when $\rot=\Omega(1/T^{3/4})$, the additional regret from learning the maximum rotting rate is negligible compared with the regret from the rotting of rewards. It is an open problem to achieve the optimal regret bound for any value of $\rot$, without knowing the value of this parameter.

\end{remark}

\section{Numerical experiments}
\label{sec:num}

\begin{figure}[t]
\centering
\begin{subfigure}[b]{.45\textwidth}
\includegraphics[width=\linewidth]{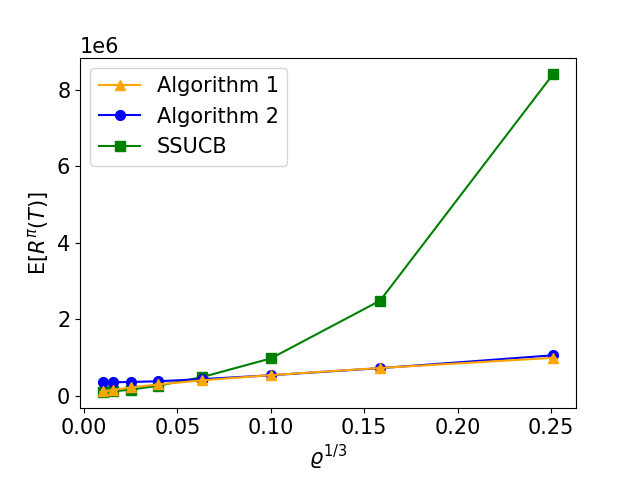}\caption{}\end{subfigure}
\begin{subfigure}[b]{.45\textwidth}\includegraphics[width=\linewidth]{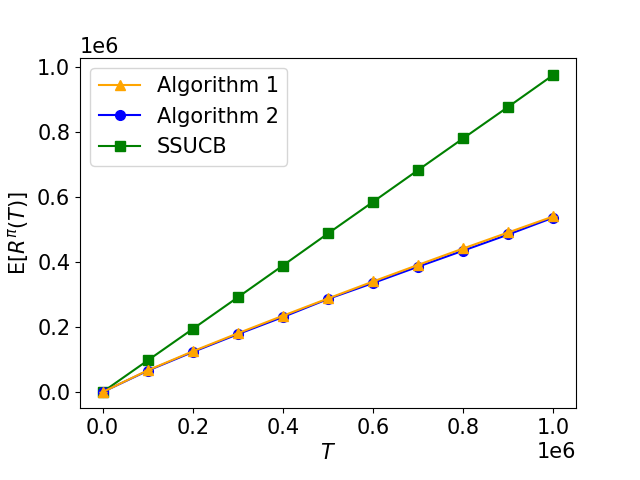}\caption{}\end{subfigure}
\begin{subfigure}[b]{.45\textwidth}\includegraphics[width=\linewidth]{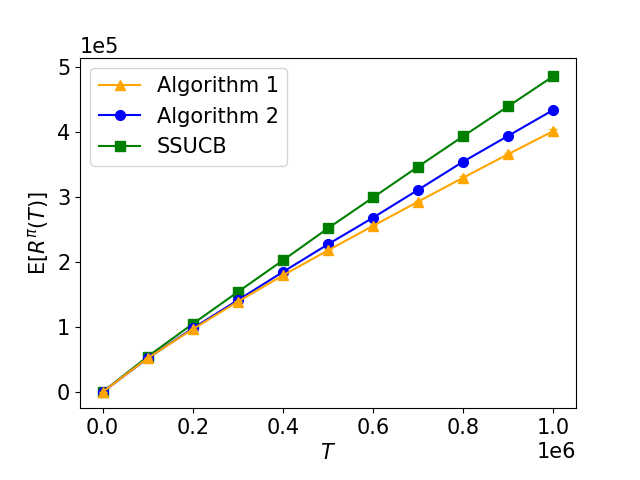}\caption{}\end{subfigure}
\begin{subfigure}[b]{.45\textwidth}
\includegraphics[width=\linewidth]{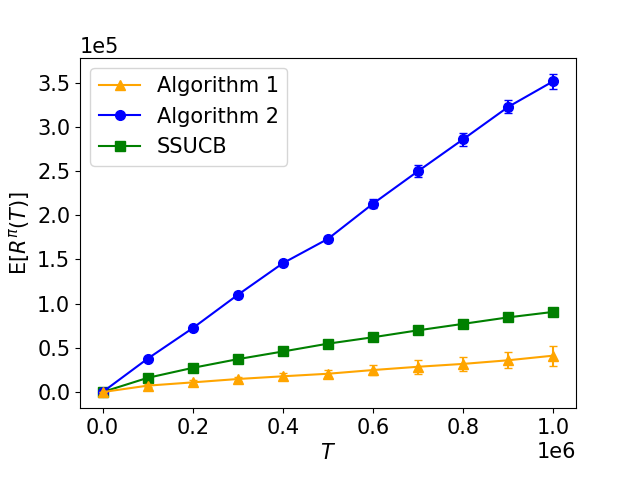}\caption{}\end{subfigure}

\caption{Performance of Algorithms~\ref{alg:Alg1} and \ref{alg:Alg2}, and SSUCB: (a) regret versus $\rot^{1/3}$ for fixed time horizon $T$, and (b,c,d) regret versus time horizon $T$ with rotting rates $\rot=1/T^{1/2}$ (b), $1/T^{3/5}$ (c), and $1/T^{3/2}$ (d).  
}
\label{fig:1}
\end{figure}

In this section we present results of our numerical experiments using synthetic data in the rotting setting with infinitely many arms\footnote{Our code is available at \url{https://github.com/junghunkim7786/rotting_infinite_armed_bandits}}.
To the best of our knowledge, there are no previously-proposed algorithms for the setting of rotting with infinitely many arms. We compare the performance of Algorithms~\ref{alg:Alg1} and \ref{alg:Alg2} with SSUCB \cite{Bayati}, which was proposed for infinitely many arms with stationary rewards, and is known to have a near-optimal regret for stationary sub-Gaussian reward distributions. 
From the theoretical results in Section~\ref{sec:lb} and Section~\ref{sec:algo}, we expect that Algorithm~\ref{alg:Alg1} and SSUCB would have similar performance when the maximum rotting rate is sufficiently small, which may be regarded as a nearly stationary case, and that both Algorithm~\ref{alg:Alg1} and Algorithm~\ref{alg:Alg2} would outperform SSUCB when the rotting rate is sufficiently large. We expect that Algorithm~\ref{alg:Alg2} would not be competitive in the nearly stationary case because of the extra regret from the rotting rate estimation. It requires the rotting rate to be sufficiently large to have the near-optimal regret bound. We will confirm these insights by our numerical results in what follows.   

 In our experiments, we consider the case of identical rotting with $\rot_t=\rot$ for all $t\in[T]$. We generate initial mean rewards of arms by sampling from uniform distribution on $[0,1]$. In each time step, stochastic reward from pulling an arm has a Gaussian noise with mean zero and variance $1$. We repeat each experiment $10$ times and compute confidence intervals for confidence probability 0.95.
 
 We first investigate the performance of algorithms for varied rotting rate $\rot$ and fixed time horizon $T$. We ran experiments for rotting rate $\rot$ set to  $1/T,1/T^{0.9},1/T^{0.8},\ldots,1/T^{0.3},$ and measured the expected regret for the time horizon $T=10^6$. 
 In Figure~\ref{fig:1} (a), we can confirm that our algorithms show more robust performance than SSUCB for various rotting rates with linearly increasing regret with respect to $\rot^{1/3}$ for large rotting, which matches Theorems~\ref{thm:R_upper_bd_e} and \ref{thm:R_upper_bd_no_e}. We observe that for large rotting cases, our algorithms have similar performance and both outperform SSUCB. For sufficiently small rotting, we can observe that all the three algorithms have comparable performance while  Algorithm~\ref{alg:Alg2} performs slightly worse. These results conform to the insights derived from our theoretical analysis.
 
We next investigate the performance of algorithms versus time horizon $T$ and rotting rate $\rot$ depending on $T$. We ran experiments for time horizon $T$ taking values $1, 1\times 10^5, 2\times 10^5, \ldots, 10^6$, and measured expected regret of each case. We set the rotting rate $\rot$ to $ 1/T^{1/2}$, $ 1/T^{3/5}$ and $1/T^{3/2}$. Note that the case $1/T^{3/2}$ may be considered as a nearly-stationary case, because in this case the regret lower bound is $\Omega(\sqrt{T})$ from Theorem~\ref{thm:lower_bd_e}. In Figure~\ref{fig:1} (b) and (c), corresponding to large rotting rates, we observe that Algorithms \ref{alg:Alg1} and \ref{alg:Alg2} have similar performance and outperform SSUCB. The gaps between regrets of our algorithms and SSUCB become smaller by decreasing the rotting rate from $1/T^{1/2}$ to $1/T^{3/5}$. This is because SSUCB is designed for the case of stationary rewards and has a near-optimal regret in the stationary case. In the near-stationary case, in Figure~\ref{fig:1} (d), we observe that Algorithm~\ref{alg:Alg1} has best performance and, as expected, Algorithm~\ref{alg:Alg2} has worst performance.

\section{Conclusion}

\label{sec:conc}

In this paper we studied the infinitely many-armed bandit problem with rested rotting rewards. We provided a regret lower bound and proposed an algorithm which achieves a near-optimal regret, when the the maximum rotting rate is known to the algorithm. We also proposed an algorithm which does not require knowledge of the rotting rate and we showed that it achieves a near-optimal regret for any large enough rotting rate. 
In future work, it may be of interest to relax the assumption of uniform distribution for initial mean rewards.




\bibliography{mybib}
\bibliographystyle{ieeetr}

\newpage
\appendix
\onecolumn
\section{Appendix}




\subsection{Preliminaries}

Here we state some known concentration inequalities that we use in our proofs. 

\begin{lemma}[Theorem 6.2.35 in \cite{Alex}] Let $X_1,\dots,X_n$ be identical independent Bernoulli random variables. Then, for $0<\nu<1$, we have 
\begin{align*}
    \mathbb{P}\left(\sum_{i=1}^{n}X_i\ge(1+\nu)\mathbb{E}\left[\sum_{i=1}^{n}X_i\right]\right)\le\exp\left(-\frac{\nu^2\mathbb{E}[\sum_{i=1}^{n}X_i]}{3} \right).
\end{align*} \label{lem:chernoff_bino}
\end{lemma}

\begin{lemma}[Corollary 1.7 in \cite{rigollet2015high}]
Let $X_1,\dots,X_n$ be independent random variables with $\sigma$-sub-Gaussian distributions. Then, for any $a = (a_1, \ldots, a_n)^\top \in \mathbb{R}^n$ and $t\geq 0$,  we have
$$\mathbb{P}\left(\sum_{i=1}^na_iX_i>t\right)\le \exp\left(-\frac{t^2}{2\sigma^2\|a\|_2^2}\right) \text{ and } 
\mathbb{P}\left(\sum_{i=1}^na_iX_i<-t\right)\le \exp\left(-\frac{t^2}{2\sigma^2\|a\|_2^2}\right).
$$
\label{lem:chernoff_sub-gau}
\end{lemma}

\subsection{Proof of Theorem~\ref{thm:lower_bd_e}}\label{app:lower_bd_e}
For showing the lower bound, we will classify arms to be either bad or good depending on the value of their initial mean rewards, in a way that will be defined shortly. We will analyze the number of bad arm pulls which is the main contributor to the regret. For the analysis, we count the number of bad arm pulls until the first or $m$-th good arm is pulled over a fixed time horizon in order to preserve i.i.d. property of mean rewards in the proof. To distinguish good from bad arms, we utilize two thresholds in the proof as follows.
Let $0<\delta<c<1$ be such that $1>1-\delta> 1-c> 0$ for $\delta$, which will be set small, and some constant $c$. Then, $1-\delta$ and $1-c$ represent threshold values for distinguishing good arms and bad arms, respectively. In $\mathcal{A}$, let $\bar{a}_1,\bar{a}_2,\dots,$ be a sequence of arms. Given a policy $\pi$, without loss of generality let $\pi$ select arms following the sequence of $\bar{a}_1,\bar{a}_2,\dots,$.

\textbf{Case of small rotting:} When $\rot=O( 1/T^{3/2})$, the lower bound of order $\sqrt{T}$ for the stationary case, from Theorem 3 in \cite{Wang}, is tight enough for the non-stationary case. This is because we only need to pay the extra regret of at most of order $\sqrt{T}$ for small $\rot$. From Theorem 3 in \cite{Wang}, we have
\begin{align}
    \mathbb{E}[R^\pi(T)]=\Omega(\sqrt{T}).\label{eq:lowbd_small_e}
\end{align}
We note that even though the mean rewards are rotting in our setting, Theorem 3 in \cite{Wang} can be applied to our setting without any proof changes providing a tight regret bound for the near-stationary case. For the sake of completeness, we provide the proof of the theorem.  Let $K_1$ denote the number of ``bad" arms $a$ that satisfy $\mu_1(a)\le1-c$  before selecting the first ``good" arm, which satisfies $\mu_1(a)>1-\delta$, in the sequence of arms $\bar{a}_1,\bar{a}_2,\ldots.$  Let $\overline{\mu}$ be the initial mean reward of the best arm among the selected arms by $\pi$ over time horizon $T$. Then for some $\kappa>0$, we have 
\begin{align}
    R^\pi(T)&=R^\pi(T)\mathbbm{1}(\overline{\mu}\le 1-\delta)+R^\pi(T)\mathbbm{1}(\overline{\mu}> 1-\delta)\cr 
    &\ge T\delta \mathbbm{1}(\overline{\mu}\le 1-\delta)+K_1c\mathbbm{1}(\overline{\mu}> 1-\delta)\cr
    &\ge T\delta \mathbbm{1}(\overline{\mu}\le 1-\delta)+\kappa c\mathbbm{1}(\overline{\mu}> 1-\delta, K_1\ge \kappa).\label{eq:regret_lower_decom_small}
\end{align}
By taking expectations on the both sides in \eqref{eq:regret_lower_decom_small} and setting $\kappa=T\delta/c$, we have
\begin{align*}
    \mathbb{E}[R^\pi(T)]\ge T\delta \mathbb{P}(\overline{\mu}\le 1-\delta)+\kappa c(\mathbb{P}(\overline{\mu}>1-\delta)-\mathbb{P}(K_1<\kappa))=c\kappa \mathbb{P}(K_1\ge \kappa). 
\end{align*}
Let $\delta^\prime=\delta/(1-c+\delta)$. Then we observe that $K_1$ follows a geometric distribution with success probability $\mathbb{P}(\mu_1(a)>1-\delta)/p(\mu_1(a)\notin (1-c,1-\delta))=\delta^\prime,$ in which the success probability is the probability of selecting a good arm given that the arm is either a good or bad arm. Then by setting $\delta=1/\sqrt{T}$ with $\kappa=\sqrt{T}/c$, for some constant $C>0$ we have
\begin{align*}
    \mathbb{E}[R^\pi(T)]\ge c\kappa(1-\delta^\prime)^\kappa=\Omega\left(\sqrt{T}(1-C/\sqrt{T})^{\sqrt{T}/c} \right)=\Omega(\sqrt{T}),
\end{align*} where the last equality is obtained from $\log x\ge 1-1/x$ for all $x>0$.

\textbf{Case of large rotting:} When $\rot=\omega(1/T^{3/2})$, however, the lower bound of the stationary case is not tight enough. Here we provide the proof of the lower bound $T\rot^{1/3}$ for the case of $\rot=\omega(1/T^{3/2})$. 
Let $K_m$ denote the number of ``bad" arms $a$ that satisfy $\mu_1(a)\le1-c$  before selecting $m$-th ``good" arm, which satisfies $\mu_1(a)>1-\delta$, in the sequence of arms $\bar{a}_1,\bar{a}_2,\ldots.$ Let $N_T$ be the number of selected good arms $a$ such that $\mu_1(a)>1-\delta$ until $T$. 

We can decompose $R^\pi(T)$ into two parts as follows: 
\begin{align}
R^\pi(T)=R^\pi(T)\mathbbm{1}(N_T<m)+R^\pi(T)\mathbbm{1}(N_T\ge m).\label{eq:R_decom_lower_e}    
\end{align}

We set $m=\lceil (1/2)T\rot^{2/3}\rceil$ and $\delta=\rot^{1/3}$with $\rot=o(1)$. For getting a lower bound for the first term in \eqref{eq:R_decom_lower_e}, $R^\pi(T)\mathbbm{1}(N_T<m)$, we use the fact that the minimal regret is obtained from the situation where there are $m-1$ arms whose mean rewards are $1$. Then we can think of the optimal policy that selects the best $m-1$ arms until their mean rewards become below the threshold $1-\delta$ (step 1) and then selects the best arm at each time for the remaining time steps (step 2). The required number of pulling each arm for the best $m-1$ arms until the mean reward becomes below $1-\delta$ is upper bounded by $\delta/\rot+1$. Therefore, the regret from step 2 is $R=\Omega((T-m\delta/\rot)\delta)=\Omega(T\rot^{1/3})$ in which the optimal policy pulls arms which  mean rewards are below $1-\delta$ for the remaining time after step 1.
Therefore, we have
\begin{align}
    R^\pi(T)\mathbbm{1}(N_T<m)=\Omega(R\mathbbm{1}(N_T<m))=\Omega(T\rot^{1/3}\mathbbm{1}(N_T<m)).\label{eq:lowbd_eq1_e}
\end{align}
For getting a lower bound of the second term in \eqref{eq:R_decom_lower_e}, $R^\pi(T)\mathbbm{1}(N_T\ge m)$, we use the minimum number of selected arms $a$ that satisfy $\mu_1(a)\le 1-c.$ When $N_T\ge m$ and $K_m\ge \kappa$, the policy selects at least $\kappa$ number of distinct arms $a$ satisfying $\mu_1(a)\le 1-c$ until $T$. Therefore, we have
\begin{align}
    R^\pi(T)\mathbbm{1}(N_T\ge m)\ge c\kappa\mathbbm{1}(N_T\ge m,K_m\ge \kappa).\label{eq:lowbd_eq2_e}
\end{align}
Let $\delta^\prime=\delta/(1-c+\delta)$. By setting $\kappa=m/\delta^\prime-m-\sqrt{m}/\delta^\prime$, with $\rot=o(1)$, we have
\begin{align}
\kappa=\Theta(T\rot^{1/3}).\label{eq:kappa}
\end{align}
Then from \eqref{eq:lowbd_eq1_e}, \eqref{eq:lowbd_eq2_e}, and \eqref{eq:kappa}, we have
\begin{align}
  \mathbb{E}[R^\pi(T)]=\Omega(T\rot^{1/3}\mathbb{P}(N_T<m)+T\rot^{1/3}\mathbb{P}(N_T\ge m,K_m\ge \kappa))\ge\Omega( T\rot^{1/3}\mathbb{P}(K_m\ge \kappa)). \label{eq:lowbd_eq3_e}
\end{align}
Next we provide a lower bound for $\mathbb{P}(K_m\ge \kappa).$ Observe that $K_m$ follows a negative binomial distribution with $m$ successes and the success
probability $\mathbb{P}(\mu_1(a)>1-\delta)/\mathbb{P}(\mu_1(a)\notin(1-c,1-\delta))=\delta/(1-c+\delta)=\delta^\prime$, in which the success probability is the probability of selecting a good arm given that the arm is either a good or bad arm. In the following lemma, we provide a concentration inequality for $K_m$.
\begin{lemma}
For any $1/2+\delta^\prime/m<\alpha <1$,
\begin{align}
    \mathbb{P}(K_m\ge\alpha m(1/\delta^\prime)-m)\ge1-\exp(-(1/3)(1-1/\alpha )^2(\alpha m-\delta^\prime)).
\end{align} \label{lem:concen_geo}
\end{lemma}
\begin{proof}
Let $X_i$ for $i>0$ be i.i.d. Bernoulli random variables with success probability $\delta^\prime.$ From Section 2 in \cite{brown2011wasted}, we have 
\begin{align}
    \mathbb{P}\left(K_m\le \left\lfloor\alpha m \frac{1}{\delta^\prime}\right\rfloor-m\right)
    =\mathbb{P}\left(\sum_{i=1}^{\left\lfloor\alpha m \frac{1}{\delta^\prime}\right\rfloor}X_i\ge m\right).\label{eq:geo_binom}
\end{align}

From \eqref{eq:geo_binom} and Lemma~\ref{lem:chernoff_bino}, for any $1/2+\delta^\prime/m<\alpha<1$ we have
\begin{align*}
    \mathbb{P}\left(K_m\le \alpha m \frac{1}{\delta^\prime}-m\right)&=\mathbb{P}\left(K_m\le\left\lfloor\alpha m \frac{1}{\delta^\prime}\right\rfloor-m\right)\cr \cr &
    =\mathbb{P}\left(\sum_{i=1}^{\left\lfloor\alpha m \frac{1}{\delta^\prime}\right\rfloor}X_i\ge m\right)\cr 
    &\le \exp\left(-\frac{(1-1/\alpha)^2}{3}\left\lfloor\alpha m \frac{1}{\delta^\prime}\right\rfloor\delta^\prime\right)\cr
    &\le \exp\left(-\frac{(1-1/\alpha)^2}{3}(\alpha m-\delta^\prime)\right),
\end{align*}
in which the first inequality comes from Lemma~\ref{lem:chernoff_bino},
which concludes the proof.
\end{proof}
From Lemma~\ref{lem:concen_geo} with $\alpha=1-1/\sqrt{m}$ and large enough $T$, we have
\begin{align}
    \mathbb{P}(K_m\ge \kappa)&\ge 1-\exp\left(-\frac{1}{3}(m-\sqrt{m}-\delta^\prime)\left(\frac{1}{\sqrt{m}-1}\right)^2\right)\cr
    &\ge 1-\exp\left(-\frac{1}{6}(m-\sqrt{m})\left(\frac{1}{\sqrt{m}-1}\right)^2\right)\cr
    &=1-\exp\left(-\frac{1}{6}\frac{\sqrt{m}}{\sqrt{m}-1}\right)\cr
    &\ge  1-\exp(-1/6). \label{eq:lowbd_eq4_e}
\end{align}
Therefore, from \eqref{eq:lowbd_eq3_e} and \eqref{eq:lowbd_eq4_e}, we have
\begin{align}
\mathbb{E}[R^\pi(T)]=\Omega(T\rot^{1/3}).    \label{eq:lowbd_large_e}
\end{align}
Finally, from \eqref{eq:lowbd_small_e} and \eqref{eq:lowbd_large_e} we conclude that for any policy $\pi$, we have $$\mathbb{E}[R^\pi(T)]=\Omega(\max\{T\rot^{1/3},\sqrt{T}\}).$$

\subsection{Proof of Theorem~\ref{thm:R_upper_bd_e}}
\label{sec:proof_thm_R_upper_bd_e}

Observe that initial mean rewards of sampled arms are i.i.d. with a uniform distribution which should simplify analysis of the expected regret. However, by fixing the number of sampled arms by a policy over the time horizon $T$, the mean rewards of arms become dependent. To deal with this dependence, we analyze the regret by controlling the number of distinct sampled arms instead of fixing the time horizon. We explain this in more details in the following proofs. 

We set $\delta=\max\{\rot^{1/3},1/\sqrt{T}\}$. Let $\Delta_1(a)=1-\mu_1(a)$. Then we define an arm $a$ to be a \emph{good} arm if $\Delta_1(a)\le \delta/2$, and, otherwise, $a$ is a \emph{bad} arm. In $\mathcal{A}$, let $\bar{a}_1,\bar{a}_2,\dots,$ be a sequence of arms, which have i.i.d. mean rewards with uniform distribution on $[0,1]$. Given a policy sampling arms in the sequence order,
let $m^\ga$ be the number of selections of distinct good arms and $m^{\ba}_i$ be the number of consecutive selections of distinct bad arms between the $i-1$-st and $i$-th selection of a good arm among $m^\ga$ good arms. We refer to the period starting from sampling the $i-1$-st good arm before sampling the $i$-th good arm as the $i$-th \emph{episode}.
Observe that $m^\ba_1,\ldots, m^\ba_{m^\ga}$ are i.i.d. random variables with geometric distribution with parameter $\delta/2$, given a fixed value of $m^\ga$. Therefore, for non-negative integer $k$ we have $\mathbb{P}(m^\ba_i=k)=(1-\delta/2)^k\delta/2$, for $i = 1, \ldots, m^\ga$. Define $\tilde{m}$ to be the number of episodes from the policy $\pi$ over the horizon $T$, $\tilde{m}^\ga$ to be the total number of selections of a good arm by the policy $\pi$ over the horizon $T$ such that $\tilde{m}^\ga=\tilde{m}$ or $\tilde{m}^\ga=\tilde{m}-1$, and $\tilde{m}_i^\ba$ to be the number of selections of a bad arm in the $i$-th episode by the policy $\pi$ over the horizon $T$. Without loss of generality, we assume that the policy selects arms in the sequence of $\bar{a}_1,\bar{a}_2,\dots,.$ Let $\mathcal{A}_T$ be the set of sampled arms over the horizon of $T$ time steps, which satisfies $|\mathcal{A}_T|\le T$. Let 
\begin{align*}
  \hat{\mu}_t(a)=\frac{\sum_{s=1}^{t-1}r_s\mathbbm{1}(a_s=a)}{n_t(a)} \text{ and } \bar{\mu}_t(a)=\frac{\sum_{s=1}^{t-1}\mu_s(a)\mathbbm{1}(a_s=a)}{n_t(a)}.  
\end{align*}
We define the event $E_1=\{|\hat{\mu}_t(a)-\bar{\mu}_t(a)|\le \sqrt{2\log(T^5)/n_t(a)} \hbox{ for all } t\in [T], a\in\mathcal{A}_T\}$ to guarantee that the estimators of initial mean reward are well estimated. Using Lemma~\ref{lem:chernoff_sub-gau}, we have
\begin{align*}
P\left(\left|\hat{\mu}_t(a)-\bar{\mu}_{t}(a)\right|\ge \sqrt{\frac{10\log T}{n_{t}(a)}}\right) \le \frac{2}{T^4}.
\end{align*}
Using union bound for $t$ and $a$, we have $\mathbb{P}(E_1^c)\le 2/T^2$.
 Recall that $R^\pi(T)=\sum_{t=1}^T1-\mu_t(a_t)$. It is true that $R^\pi(T)=o(T^2)$ because the maximum mean reward gap from rotting is bounded by $1+T\rot=o(T).$  Given that $E_1$ does not hold, the regret is $\mathbb{E}[R^\pi(T)|E_1^c]\mathbb{P}(E_1^c)=o(1)$, which is negligible comparing with the regret when $E_1$ holds true which we show later. Therefore, in the rest of the proof we assume that $E_1$ holds true.

Under a policy $\pi$, let $R_i^\ga$ be the regret (summation of mean reward gaps) contributed by pulling the good arm in the $i$-th episode and $R_{i,j}^\ba$ be the regret contributed by pulling the $j$-th bad arm in the $i$-th episode. Then let $R^{\pi}_{m^\ga}=\sum_{i=1}^{m^\ga}(R_i^\ga+\sum_{j\in[m_i^\ba]}R_{i,j}^\ba),$\footnote{Note that $R_{m^\ba}^\pi$ does not contain undefined $R_{i,j}^\ba$ such that $R_{i,j}^\ba$ when $m_i^\ba=0$.} which is the regret over the period of $m^\ga$ episodes. 
For obtaining a regret bound, we first focus on finding a required number of episodes, $m^{\ga}$, such that $R^\pi(T)\le R^{\pi}_{m^\ga}$. Then we provide regret bounds for each bad arm and good arm in an episode. Lastly, we obtain a regret bound for $\mathbb{E}[R^\pi(T)]$ using the episodic regret bound.

 For $i\in [\tilde{m}^\ga]$, $j\in [\tilde{m}_i^\ba]$, let $a(i)$ be the sampled arm for the $i$-th good arm and $a(i,j)$ be a sampled arm for $j$-th bad arm in the $i$-th episode. Then $n_T(a(i))$ is the number of pulls of the good arm in the $i$-th episode and $n_T(a(i,j))$ is the number of pulls of the $j$-th bad arm in the $i$-th episode by the policy $\pi$ over the horizon $T$. Let $\tilde{a}$ be the last sampled arm over time horizon $T$ by $\pi$. We denote by $\hat{m}^\ga$ and $\hat{m}^\ba_i$ for $i\in[\tilde{m}]$ the number of arms excluding $\tilde{a}$ in the sampled $\tilde{m}^\ga$ number of good arms and $\tilde{m}^\ba_i$ number of bad arms for $i\in[\tilde{m}]$ as follows:  \begin{equation*}
    \hat{m}^\ga=
    \begin{cases}
    \tilde{m}^\ga-1 &\text{if  $\tilde{a}$ is a good arm}  \\
     \tilde{m}^\ga& \text{otherwise}
    \end{cases},
\end{equation*}
\begin{equation*}
        \hat{m}_i^\ba=\tilde{m}_i^\ba \text{ for $i\in[\tilde{m}-1]$},
    \text{ and }
    \hat{m}_{\tilde{m}}^\ba=
    \begin{cases}
    \tilde{m}_{\tilde{m}}^\ba  &
    \text{if $\tilde{a}$ is a good arm}  \\
     \tilde{m}_{\tilde{m}}^\ba-1 & \text{otherwise}.
    \end{cases}
\end{equation*}
Those notations are defined only if they exist.\footnote{$n_T(a(i))$,  $n_T(a(i,j))$, and $\hat{m}_i^\ba$ are not defined for $i\in[0]$ or $j\in[0]$.}
 Excluding the last arm $\tilde{a}$ which the policy $\pi$ may stop to pull suddenly by reaching the horizon $T$, we provide lower bounds of the number of pulling a good arm, $n_T(a(i))$ for $i\in[\hat{m}^\ga]$ in the following lemma if they exist. 
  
 \begin{lemma}
  Under $E_1$, given $\hat{m}^\ga$, for any $i\in[\hat{m}^\ga]$ we have $$n_T(a(i))\ge\delta/(2\rot).$$
  \label{lem:n_low_bd}
 \end{lemma}
 \begin{proof}
 Let $$\hat{\mu}_t^o(a)=\frac{\sum_{s=1}^{t-1}(r_s+\rot_sn_s(a))\mathbbm{1}(a_s=a)}{n_t(a)},$$ which satisfies $\tilde{\mu}_t^o(a)\ge \hat{\mu}_t^o(a)$ from $\rot_s\le \rot$ for all $s$. Under $E_1$, we can easily show that for all $t\in[T]$ and $a\in\mathcal{A}_T$, 
 $$|\hat{\mu}^o_t(a)-\mu_1(a)|\le \sqrt{10\log(T)/n_t(a)}.$$
 Let $a(i)$ be a sampled good arm in the $i$-th episode. Suppose that $n_t(a(i))=\lfloor\delta/(2\rot)\rfloor$ for some $t>0$, then we have \begin{align*}
    \tilde{\mu}_t^o(a(i))-\rot n_t(a(i))+\sqrt{10\log(T)/n_T(a(i))}&\ge\hat{\mu}_t^o(a(i))-\rot n_t(a(i))+\sqrt{10\log(T)/n_T(a(i))}\cr &\ge \mu_1(a(i))-\delta/2\cr &\ge 1-\delta,
\end{align*}
where the second inequality is obtained from $E_1$ and  $n_t(a(i))\le \delta/(2\rot)$, and the third inequality is from $\mu_1(a(i))\ge 1- \delta/2$. Therefore, policy $\pi$ must pull arm $a$ more times than $\lfloor\delta/(2\rot)\rfloor$, which implies $n_T(a(i))\ge\delta/(2\rot)$ for any $i\in[\hat{m}^\ga]$. 
 \end{proof}

\textbf{We first consider the case where $\rot=\omega(1/T^{3/2})$.} We have $\delta=\rot^{1/3}$. For getting $R^{\pi}_{m^\ga}$, here  we define the policy $\pi$ after time $T$ such that it pulls $\lceil\delta/(2\rot)\rceil$ amount for a good arm and $0$ for a bad arm. We note that defining how $\pi$ works after $T$ is only for the proof to get a regret bound over time horizon $T$. We define $n_T(a(i))=0$ for $i\in[m^\ga]/[\tilde{m}^\ga]$ for convenience. For the last arm $\tilde{a}$ over the horizon $T$, it pulls the arm  up to $\max\{\lceil\delta/(2\rot)\rceil,n_T(\tilde{a})\}$ amounts if $\tilde{a}$ is a good arm.
For $i\in[m^\ga]$, $j\in[m_i^\ba]$ let $n_i^\ga$ and $n_{i,j}^\ba$ be the number of pulling the good arm in $i$-th episode and $j$-th bad arm in $i$-th episode from the policy, respectively. Here we define $n_i^\ga$'s and $n_{i,j}^\ba$'s as follows:

If $\tilde{a}$ is a good arm,
\begin{equation*}
    n_i^\ga=
    \begin{cases}
    n_T(a(i)) &\text{for } i\in[\tilde{m}^\ga-1]  \\
     \max\{\lceil\delta/(2\rot)\rceil,n_T(a(i))\}& \text{for } i\in[m^\ga]/[\tilde{m}^\ga-1]
    \end{cases}, 
     \end{equation*}
    \begin{equation*}
    n_{i,j}^\ba=
    \begin{cases}
    \tilde{n}_T(a(i,j))&\text{for } i\in[\tilde{m}^\ga],j\in[\tilde{m}_i^\ba]\\
    0 &\text{for } i\in[m^\ga]/[\tilde{m}^\ga],j\in[m^\ba_i]/[\tilde{m}_i^\ba].
    \end{cases}
\end{equation*}

Otherwise,
\begin{equation*}
    n_i^\ga=
    \begin{cases}
    n_T(a(i)) &\text{for } i\in[\tilde{m}^\ga]  \\
     \lceil\delta/(2\rot)\rceil& \text{for } i\in[m^\ga]/[\tilde{m}^\ga]
    \end{cases}, 
    n_{i,j}^\ba=
    \begin{cases}
    \tilde{n}_T(a(i,j))&\text{for } i\in[\tilde{m}^\ga],j\in[\tilde{m}_i^\ba]\\
    0 &\text{for } i\in[m^\ga]/[\tilde{m}^\ga-1],j\in[m^\ba_i]/[\tilde{m}_i^\ba].
    \end{cases}
\end{equation*}
We note that $n_{i}^\ga$'s and $n_{i,j}^\ba$'s are defined only if they exist.\footnote{$n_i^\ga$ and $n_{i,j}^\ba$ are not defined for $i\in[0]$ or $j\in[0]$.}
Then we provide $m^{\ga}$ such that $R^\pi(T)\le R^{\pi}_{m^\ga}$ in the following lemma.
\begin{lemma}
Under $E_1$, when $m^\ga=\lceil 2T\rot^{2/3}\rceil$ we have 
$$R^\pi(T)\le R^{\pi}_{m^\ga}.$$\label{lem:regret_bd_prob}
\end{lemma}
\begin{proof}
From Lemma~\ref{lem:n_low_bd}, with $\delta=\rot^{1/3}$
we have 
$$\sum_{i\in[m^\ga]}\left(n^\ga_i+\sum_{j\in[m^\ba_i]}n^\ba_{i,j}\right)\ge m^\ga\frac{\delta}{2\rot}\ge T,$$
which implies that $R^\pi(T)\le R^{\pi}_{m^\ga}.$ 

\end{proof}
From the result of Lemma~\ref{lem:regret_bd_prob}, we set $m^\ga=\lceil 2T\rot^{2/3}\rceil$.
  For getting a bound for $\mathbb{E}[R^{\pi}_{m^\ga}]$, we provide bounds for $\mathbb{E}[R^\ga_i]$ and $\mathbb{E}[R^\ba_{i,j}]$ in the following lemma.
  \begin{lemma}
   Under $E_1$ and policy $\pi$, for any $i\in[m^\ga]$, $j\in[m^\ba_i]$, we have 
\begin{align*}
    \mathbb{E}[R_i^\ga]
    =\tilde{O}\left(\frac{\delta}{\rot^{2/3}}+\frac{\delta^2}{\rot}+\frac{1}{\rot^{1/3}}\right),
\end{align*}
and 
\begin{align*}
    \mathbb{E}[R_{i,j}^\ba]= \tilde{O}\left(1+\frac{\delta}{\rot^{1/3}}+\frac{\delta^2}{\rot^{2/3}}\right).
\end{align*}\label{lem:R_good_bad_bd}
  \end{lemma}
  \begin{proof}
   First we provide a bound for $R^\ga_i$ using an upper bound of $n^\ga_i$. Recall that $a(i)$ is the sampled arm for the $i$-th good arm.
   We have
   \begin{align*}
     \tilde{\mu}_t^o(a(i))-\rot n_t(a(i))&=\hat{\mu}_t(a(i))+\rot\frac{\sum_{s=1}^{t-1}n_s(a(i))\mathbbm{1}(a_s=a(i))}{n_t(a(i))}-\rot n_t(a(i))\cr &
     \le \hat{\mu}_t(a(i))+\rot \frac{n_t(a(i))+1}{2}-\rot n_t(a(i))\cr&
     =\hat{\mu}_t(a(i))-\rot \frac{n_t(a(i))-1}{2}.
   \end{align*}
   Then since under $E_1$,
   \begin{align*}
   \hat{\mu}_t(a(i))-(\rot/2)(n_t(a(i))-1)&\le \bar{\mu}_t(a(i))+\sqrt{10\log(T)/n_t(a(i))}-(\rot/2)(n_t(a(i))-1)\cr &\le  1+\sqrt{10\log(T)/n_t(a(i))}-(\rot/2)(n_t(a(i))-1),
   \end{align*}
   for $i\in[\tilde{m}^\ga]$, from the policy $\pi$, we need to get $n$ such that
\begin{align}
    1-\frac{\rot}{2}(n-1) +2\sqrt{10\log(T)/n}< 1-\delta,\label{eq:n_good_con_up}
\end{align} in which $n+1$ is an upper bound for $n^\ga_i$. Let $n_1=2(\delta+\rot^{1/3})/\rot+1$ and $n_2=C\log(T)/\rot^{2/3}$ with some large enough constant $C>0$. Then $n=n_1+n_2$ satisfies \eqref{eq:n_good_con_up} because $1-\rot (n_1-1)/2+2\sqrt{10\log(T)/n_2}< 1-\delta.$ Therefore, for all $i\in[\tilde{m}^\ga]$ we have 
$n_i^\ga=\tilde{O}((\delta+\rot^{1/3})/\rot).$
Then with the fact that $n_i^\ga=\lceil\delta/(2\rot)\rceil$ for $i\in[m^\ga]/[\tilde{m}^\ga]$ if they exist, for any $i\in[m^\ga]$ we have $$n_i^\ga=\tilde{O}((\delta+\rot^{1/3})/\rot).$$  Then for any $i\in[m^\ga]$ we have 
\begin{align*}
    \mathbb{E}[R_i^\ga]&\le \mathbb{E}
    \left[\Delta_1(a(i))n_i^\ga+\frac{n_i^\ga(n_i^\ga-1)}{2}\rot\right]\cr
    &=\tilde{O}\left(\frac{2}{\delta}\int_0^{\delta/2}\frac{\delta+\rot^{1/3}}{\rot}x+\left(\frac{\delta+\rot^{1/3}}{\rot}\right)^2\rot dx\right)\cr
    &=\tilde{O}\left(\frac{\delta}{\rot^{2/3}}+\frac{\delta^2}{\rot}+\frac{1}{\rot^{1/3}}\right),
\end{align*}
where the first equality is obtained from the fact that $\Delta_1(a(i))$ are i.i.d. random variables with uniform distribution on $[0,\delta/2]$ and $n_i^\ga=\tilde{O}((\delta+\rot^{1/3})/\rot)$.

Now we provide an upper bound of $n^\ba_{i,j}$ to get a bound of $R_{i,j}^\ba$ for $i\in[m^\ga],j\in[m_i^\ba]$. Let $a(i,j)$ be a sampled arm for $j$-th bad arm in the $i$-th episode. When $\delta/2<\Delta_1(a(i,j))\le \delta+\rot^{1/3}$, as in the case of the good arm, $n=2(\delta+\rot^{1/3})/\rot+1 +C_1\log(T)/\rot^{2/3}$ for some large enough constant $C_1>0$, satisfies $\eqref{eq:n_good_con_up}$ so that $n_{i,j}^\ba=\tilde{O}((\delta+\rot^{1/3})/\rot)$ for all $i\in[\tilde{m}], j\in[\tilde{m}_i^\ba]$.
Since under $E_1$ 
\begin{align*}
   \hat{\mu}_t(a(i,j))-(\rot/2)(n_t(a(i,j))-1)&\le \bar{\mu}_t(a(i,j))+\sqrt{10\log(T)/n_t(a(i,j))}-(\rot/2)(n_t(a(i,j))-1)\cr &\le  \mu_1(a(i,j))+\sqrt{10\log(T)/n_t(a(i,j))}-(\rot/2)(n_t(a(i,j))-1),
   \end{align*}
when $\delta+\rot^{1/3}<\Delta_1(a(i,j))\le 1$, from the policy $\pi$ under $E_1$, we need to get $n\ge 1$ such that
\begin{align}
    \mu_1(a(i,j))-\frac{\rot}{2} (n-1)+2\sqrt{10\log(T)/n}< 1-\delta, \label{eq:n_bad_con_up}
\end{align}
in which $n+1$ is an upper bound of $n_{i,j}^\ba$ for $i\in[\tilde{m}], j\in[\tilde{m}_i^\ba]$. From a sufficient condition for \eqref{eq:n_bad_con_up} to hold such that $$\mu_1(a(i,j))+2\sqrt{10\log(T)/n}< 1-\delta,$$ we can find that $n=C_2\log(T)/(\Delta_1(a(i,j))-\delta)^2$
for some large constant $C_2>0$ satisfies \eqref{eq:n_bad_con_up}. Therefore, when $\delta+\rot^{1/3}<\Delta_1(a(i,j))\le 1$, for all $i\in[\tilde{m}],j\in[\tilde{m}_i^\ba]$ we have $n_{i,j}^\ba=\tilde{O}(1/(\Delta_1(a(i,j))-\delta)^2).$ Then with the fact that $n_{i,j}^\ba=0$ for $i\in[m^\ga]/[\tilde{m}^\ga]$, $j\in[m_i^\ba]/[\tilde{m}_i^\ba]$ if they exist, for any $i\in[m^\ga]$ and $j\in[m^\ba_i]$, we have
\begin{equation*}
    n_{i,j}^\ba=
    \begin{cases}
    \tilde{O}((\delta+\rot^{1/3})/\rot) & \text{if $\delta/2<\Delta_1(a(i,j))\le\delta+\rot^{1/3}$}\\
    \tilde{O}(1/(\Delta_1(a(i,j))-\delta)^2) & \text{if $\delta+\rot^{1/3}<\Delta_1(a(i,j))\le 1$}
    \end{cases}
\end{equation*}
Then for any $i\in[m^\ga]$, $j\in[m^\ba_i]$, we have 
\begin{align*}
    \mathbb{E}[R_{i,j}^\ba]&\le\mathbb{E}
    \left[\Delta_1(a(i,j))n_{i,j}^\ba+\frac{n_{i,j}^\ba(n_{i,j}^\ba-1)}{2}\rot\right]\cr
    &=\tilde{O}\left(\frac{1}{1-\delta/2}\left(\int_{\delta/2}^{\delta+\rot^{1/3}}\frac{\delta+\rot^{1/3}}{\rot}x+\left(\frac{\delta+\rot^{1/3}}{\rot}\right)^2\rot dx\right . \right . + \left . \left .\int_{\delta+\rot^{1/3}}^1 \frac{1}{(x-\delta)^2}x+\frac{1}{(x-\delta)^4}\rot dx \right)\right)\cr
    &= \tilde{O}\left(1+\frac{\delta}{\rot^{1/3}}+\frac{\delta^2}{\rot^{2/3}}\right).
\end{align*}
  \end{proof}
 
Recall that $R^{\pi}_{m^\ga}=\sum_{i=1}^{m^\ga}(R_i^\ga+\sum_{j\in[m_i^\ba]}R_{i,j}^\ba).$ With $\delta=\rot^{1/3}$ and $m^\ga=\lceil 2T\rot^{2/3}\log(1/\delta^\prime)\rceil$, from Lemmas ~\ref{lem:regret_bd_prob} and \ref{lem:R_good_bad_bd}, and the fact that $m_i^\ba$'s are i.i.d. random variables with geometric distribution with $\mathbb{E}[m_i^\ba]=2/\delta-1$, we have
\begin{align}
\mathbb{E}[R^\pi(T)]&=O(\mathbb{E}[R^{\pi}_{m^\ga}])\cr &=O\left(\mathbb{E}\left[\sum_{i=1}^{m^\ga}\left(R^\ga_i+\sum_{j\in[m^\ba_i]}R^\ba_{i,j}\right)\right]\right)\cr
  &= \tilde{O}\left(T\rot^{2/3}\left(\left(\frac{\delta}{\rot^{2/3}}+\frac{\delta^2}{\rot}+\frac{1}{\rot^{1/3}}\right)+\frac{1}{\delta}\left(1+\frac{\delta}{\rot^{1/3}}\right)\right)\right)\cr
  &=\tilde{O}\left(T\rot^{2/3}\left(\frac{\delta}{\rot^{2/3}}+\frac{\delta^2}{\rot}+\frac{1}{\delta}+\frac{1}{\rot^{1/3}}\right)\right)\cr
  &=\tilde{O}\left(T\rot^{1/3}\right).\label{eq:regret_bd_large_e}
\end{align}

\textbf{Now we consider the case where $\rot=O( 1/T^{3/2})$.} We have  $\delta=\Theta(1/\sqrt{T})$. With a slight abuse of notation, we use $\pi$ for a  modified strategy after $T$.
For getting $R^{\pi}_{m^\ga}$, here we define the policy $\pi$ after time $T$ such that it pulls $T$ amounts for a good arm and once for a bad arm. For the last arm $\tilde{a}$ over the horizon $T$, it pulls the arm  up to $T$ amounts if $\tilde{a}$ is a good arm. With slight abuse of notation,
for $i\in[m^\ga]$, $j\in[m_i^\ba]$ let $n_i^\ga$ and $n_{i,j}^\ba$ be the number of pulling the good arm in $i$-th episode and $j$-th bad arm in $i$-th episode from the policy, respectively. Here we define $n_i^\ga$'s and $n_{i,j}^\ba$'s as follows: 

If $\tilde{a}$ is a good arm,
\begin{equation*}
    n_i^\ga=
    \begin{cases}
    n_T(a(i)) &\text{for } i\in[\tilde{m}^\ga-1]  \\
     T & \text{for } i\in[m^\ga]/[\tilde{m}^\ga-1]
    \end{cases}, 
    n_{i,j}^\ba=
    \begin{cases}
    n_T(a(i,j))&\text{for } i\in[\tilde{m}^\ga],j\in[\tilde{m}_i^\ba]\\
    0 &\text{for } i\in[m^\ga]/[\tilde{m}^\ga],j\in[m^\ba_i]/[\tilde{m}_i^\ba].
    \end{cases}
\end{equation*}
Otherwise,
\begin{equation*}
    n_i^\ga=
    \begin{cases}
    n_T(a(i)) &\text{for } i\in[\tilde{m}^\ga]  \\
     T & \text{for } i\in[m^\ga]/[\tilde{m}^\ga]
    \end{cases}, 
    n_{i,j}^\ba=
    \begin{cases}
    n_T(a(i,j))&\text{for } i\in[\tilde{m}^\ga],j\in[\tilde{m}_i^\ba]\\
    0 &\text{for } i\in[m^\ga]/[\tilde{m}^\ga-1],j\in[m^\ba_i]/[\tilde{m}_i^\ba].
    \end{cases}
\end{equation*}


From Lemma~\ref{lem:n_low_bd}, under $E_1$ we can find that $n_i^\ga\ge \min\{\delta/(2\rot),T\}$ for $i\in[m^\ga]$.
Then if $m^\ga=C_3$ with some large enough constant $C_3>0$, then with $\delta=\Theta(1/\sqrt{T})$ and $\rot=O(1/T^{3/2})$, we have $$\sum_{i\in[m^\ga]}n_i^\ga\ge C_3\min\{\delta/(2\rot),T\}>T,$$ which implies $R^\pi(T)\le R^{\pi}_{m^\ga}$.
Therefore, we set $m^\ga=C_3$.
  For getting a bound for $\mathbb{E}[R^{\pi}_{m^\ga}]$, we provide bounds for $\mathbb{E}[R^\ga_i]$ and $\mathbb{E}[R^\ba_{i,j}]$ in the following lemma.
  \begin{lemma}
Under $E_1$ and policy $\pi$, for any $i\in[m^\ga]$ and $j\in[m^\ba_i]$, we have 
\begin{align*}
    \mathbb{E}[R_i^\ga]
    =O\left(\delta T+T^2\rot\right),
\end{align*}
and 
\begin{align*}
    \mathbb{E}[R_{i,j}^\ba]=\tilde{O}\left(T\delta^2+\sqrt{T}\delta+1\right).
\end{align*}\label{lem:R_good_bad_bd_small_e}
  \end{lemma}
  \begin{proof}
   First we provide a bound for $R^\ga_i$ using an upper bound of $n^\ga_i$. With the definition of $n_i^\ga=T$ for $i\in[m^\ga]/[\tilde{m}^\ga]$, for any $i\in[m^\ga]$ we have $$n_i^\ga\le T.$$ Recall that $a(i)$ is the sampled arm for the $i$-th good arm. Then for any $i\in[m^\ga]$ we have 
\begin{align*}
    \mathbb{E}[R_i^\ga]&\le\mathbb{E}
    \left[\Delta_1(a(i))n_i^\ga+\frac{n_i^\ga(n_i^\ga-1)}{2}\rot\right]\cr
    &=O\left(\delta T+T^2\rot \right),
\end{align*}
where the first equality is obtained from the fact that $\Delta_1(a(i))$'s are i.i.d. random variables with uniform distribution on $[0,\delta/2]$ and $n_i^\ga\le T$.

Now we provide an upper bound of $n^\ba_{i,j}$ to get a bound of $R_{i,j}^\ba$ for $i\in[m^\ga],j\in[m_i^\ba]$. Let $a(i,j)$ be a sampled arm for $j$-th bad arm in the $i$-th episode. When $\delta/2<\Delta_1(a(i,j))\le \delta+1/\sqrt{T}$, as in the case of the good arm, $n_{i,j}^\ba\le T$ for all $i\in[\tilde{m}], j\in[\tilde{m}_i^\ba]$. When $\delta+1/\sqrt{T}<\Delta_1(a(i,j))\le 1$, 
since under $E_1$
\begin{align*}
    \tilde{\mu}_t^o(a(i,j))-\rot n_t(a(i,j))&=\hat{\mu}_t(a(i,j))+\rot\frac{\sum_{s=1}^{t-1}n_s(a(i,j))\mathbbm{1}(a_s=a(i,j))}{n_t(a(i,j))}-\rot n_t(a(i,j))\cr &
     \le \hat{\mu}_t(a(i,j))+\rot \frac{n_t(a(i,j))+1}{2}-\rot n_t(a(i,j))\cr&
     =\hat{\mu}_t(a(i,j))-\rot \frac{n_t(a(i,j))-1}{2}\cr &\le \bar{\mu}_t(a(i,j))+\sqrt{10\log(T)/n_t(a(i,j))}-(\rot/2)(n_t(a(i,j))-1)\cr &\le  \mu_1(a(i,j))+\sqrt{10\log(T)/n_t(a(i,j))}-(\rot/2)(n_t(a(i,j))-1),
\end{align*}
from the policy $\pi$ we need to get $n$ such that
\begin{align}
    \mu_1(a(i,j))-(\rot/2) (n-1)+2\sqrt{10\log(T)/n}< 1-\delta, \label{eq:n_bad_con_up_small_e}
\end{align}
in which $n+1$ is an upper bound of $n_{i,j}^\ba$ for $i\in[\tilde{m}], j\in[\tilde{m}_i^\ba]$. From a sufficient condition for \eqref{eq:n_bad_con_up_small_e} to hold such that $$\mu_1(a(i,j))+2\sqrt{10\log(T)/n}< 1-\delta,$$ we can find that $n=C_4\log(T)/(\Delta_1(a(i,j))-\delta)^2$
for some large constant $C_4>0$ satisfies \eqref{eq:n_bad_con_up_small_e}. Therefore, when $\delta+1/\sqrt{T}<\Delta_1(a(i,j))\le 1$, for all $i\in[\tilde{m}],j\in[\tilde{m}_i^\ba]$ we have $n_{i,j}^\ba=\tilde{O}(1/(\Delta_1(a(i,j))-\delta)^2).$ Then with the fact that $n_{i,j}^\ba=0$ for $i\in[m^\ga]/[\tilde{m}^\ga]$, $j\in[m_i^\ba]/[\tilde{m}_i^\ba]$, for any $i\in[m^\ga]$, $j\in[m^\ba_i]$, if $\delta/2<\Delta_1(a(i,j))\le\delta+1/\sqrt{T}$, we have $$n_{i,j}^\ba\le T,$$ and if $\delta+1/\sqrt{T}<\Delta_1(a(i,j))\le 1$, we have $$n_{i,j}^\ba=\tilde{O}(1/(\Delta_1(a(i,j))-\delta)^2).$$ Then for any $i\in[m^\ga]$, $j\in[m^\ba_i]$, we have 
\begin{align*}
    \mathbb{E}[R_{i,j}^\ba]&\le\mathbb{E}
    \left[\Delta_1(a(i,j))n_{i,j}^\ba+\frac{n_{i,j}^\ba(n_{i,j}^\ba-1)}{2}\rot\right]\cr
    &=\tilde{O}\left(\frac{1}{1-\delta/2}\left(\int_{\delta/2}^{\delta+1/\sqrt{T}}Tx+T^2\rot dx\right . \right . + \left . \left .\int_{\delta+1/\sqrt{T}}^1 \frac{1}{(x-\delta)^2}x+\frac{1}{(x-\delta)^4}\rot dx \right)\right)\cr
    &= \tilde{O}\left(T\delta^2+\sqrt{T}\delta+1\right).
\end{align*}
  \end{proof}

Then with $\delta=\Theta(1/\sqrt{T})$ and $m^\ga=C_3$, we have
\begin{align}
\mathbb{E}[R^\pi(T)]&=O(\mathbb{E}[R^{\pi}_{m^\ga}])\cr
&=O\left(\mathbb{E}\left[\sum_{i\in[m^\ga]}(R_i^\ga+\sum_{j\in[m_i^\ba]}R_{i,j}^\ba)\right]\right)\cr
&=\tilde{O}\left(\left(\delta T+T^2\rot\right)+\frac{1}{\delta}\left(T\delta^2+\sqrt{T}\delta+1\right)\right)\cr
&=\tilde{O}(\sqrt{T}),    \label{eq:regret_bd_small_e}
\end{align}
where the third equality is obtained from Lemma~\ref{lem:R_good_bad_bd_small_e} and  $\mathbb{E}[m_i^\ba]=2/\delta-1$.

\textbf{Finally, we can conclude the proof:} From \eqref{eq:regret_bd_large_e} and \eqref{eq:regret_bd_small_e}, for $\rot=o(1)$, \\with $\delta=\max\{\rot^{1/3},1/\sqrt{T}\}$ we have
$$\mathbb{E}[R^\pi(T)]=\tilde{O}(\max\{T\rot^{1/3},\sqrt{T}\}).$$

\subsection{Proof of Theorem~\ref{thm:R_upper_bd_no_e}}\label{sec:R_upper_bd_no_e_proof}

Let $\pi_i(\beta)$ for $\beta \in \mathcal{B}$ denote the base policy for time steps between $(i-1)H+1$ and $i\cdot H\wedge T$ in Algorithm~\ref{alg:Alg2} using $\UCB_{i,t}(a,\beta)$ as a UCB index and $1-\beta^{1/3}$ as a threshold. Denote by $a_t^{\pi_i(\beta)}$ the pulled arm at time step $t$ by policy $\pi_i(\beta).$ Then, for $\beta^\dagger \in \mathcal{B}$, which is set later for a near-optimal policy, we have
\begin{equation}
\mathbb{E}[R^\pi(T)]=\mathbb{E}\left[\sum_{t=1}^T 1-\sum_{i=1}^{\lceil T/H\rceil}\sum_{t=(i-1)H+1}^{i\cdot H\wedge T}\mu_t(a_t^{\pi})\right] = \mathbb{E}[R_1^\pi(T)]+\mathbb{E}[R_2^\pi(T)].
\label{eq:regret_up_bd_bob}
\end{equation}

where 
$$
R_1^\pi(T) = \sum_{t=1}^T 1-\sum_{i=1}^{\lceil T/H\rceil}\sum_{t=(i-1)H+1}^{i\cdot H\wedge T}\mu_t(a_t^{\pi_i(\beta^\dagger)})
$$
and
$$
R_2^\pi(T) = \sum_{i=1}^{\lceil T/H\rceil}\sum_{t=(i-1)H+1}^{i\cdot H\wedge T}\mu_t(a_t^{\pi_i(\beta^\dagger)})-\sum_{i=1}^{\lceil T/H\rceil}\sum_{t=(i-1)H+1}^{i\cdot H\wedge T}\mu_t(a_t^{\pi}).
$$
Note that $R_1^\pi(T)$ accounts for the regret caused by the near-optimal base algorithm $\pi_i(\beta^\dagger)$'s against the optimal mean reward and $R_2^\pi(T)$ accounts for the regret caused by the master algorithm by selecting a base with $\beta\in\mathcal{B}$ at every block against the base with $\beta^\dagger$. In what follows, we provide upper bounds for each regret component. We first provide an upper bound for $\mathbb{E}[R_1^\pi(T)]$ by following the proof steps in Theorem~\ref{thm:R_upper_bd_e}. Then we provide an upper bound for $\mathbb{E}[R_2^\pi(T)]$. We set $H=\sqrt{T}$ and $\beta^\dagger$ to be a smallest value in $\mathcal{B}$ which is larger than $\max\{\rot,1/H^{3/2}\}$.

\paragraph{Upper bounding $\mathbb{E}[R_1^\pi(T)]$} We refer to the period starting from time step $(i-1) H+1$ to time step $i\cdot H\wedge T$ as the $i$-th \textit{block}.
For any $i\in\lceil T/H-1\rceil$, policy $\pi_i(\beta^\dagger)$ runs over $H$ time steps independent to other blocks so that  each block has the same expected regret and the last block has a smaller or equal expected regret than other blocks. Therefore, we focus on finding a bound on the regret from the first block equal to $\sum_{t=1}^{   H}1-\mu_t(a_t^{\pi_1(\beta^\dagger)})$. Denote by $\mathcal{A}(i)$ the set of sampled arms in the $i$-th block, which satisfies $|\mathcal{A}(i)|\le H$. For notation simplicity, we use $n_t(a)$ instead of $n_{1,t}(a)$ and $\tilde{\mu}^o_{t}(a)$ instead of $\tilde{\mu}^o_{1,t}(a)$.
 Let 
\begin{align*}
  \hat{\mu}_{t}(a)=\frac{\sum_{s=1}^{t-1}r_s\mathbbm{1}(a_s=a)}{n_t(a)} \text{ and } \bar{\mu}_t(a)=\frac{\sum_{s=1}^{t-1}\mu_s(a)\mathbbm{1}(a_s=a)}{n_t(a)}.  
\end{align*}
 Define the event $E_1=\{|\hat{\mu}_t(a)-\bar{\mu}_t(a)|\le \sqrt{10\log(H)/n_{t}(a)}, \hbox{ for all } t\in [H], a \in \mathcal{A}(1)\}$. Using Lemma~\ref{lem:chernoff_sub-gau}, 
 we have
 $$
 \mathbb{P}(E_1)\ge 1-2/H^{2}.
 $$ 
 We assume that $E_1$ holds true in what follows. Otherwise, the regret for the first block is negligible from $R^{\pi_1(\beta^\dagger)}(H)=o(H^2)$. The proof follows similar steps as in the proof of  Theorem~\ref{thm:R_upper_bd_e}. 
 
From $\pi_1(\beta^\dagger)$ we have $\delta={\beta^\dagger}^{1/3}$. Let $\Delta_1(a)=1-\mu_1(a)$. Then we define an arm $a$ to be a \emph{good} arm if $\Delta_1(a)\le \delta/2$, and, otherwise, $a$ is a \emph{bad} arm. In $\mathcal{A}$, let $\bar{a}_1,\bar{a}_2,\ldots$ be a sequence of arms, which have i.i.d. mean rewards with uniform distribution on $[0,1]$. Given a policy sampling arms in the sequence order, let $m^\ga$ be the number of selections of distinct good arms and $m^{\ba}_i$ be the number of consecutive selections of distinct bad arms between the $i-1$-st and $i$-th selection of a good arm among $m^\ga$ good arms. We refer to the period starting from sampling the $i-1$-st good arm before sampling the $i$-th good arm as the $i$-th \emph{episode}.
Observe that $m^\ba_1,\ldots, m^\ba_{m^\ga}$'s are i.i.d. random variables with geometric distribution with parameter $\delta/2$, conditional on the value of $m^\ga$. Therefore, $\mathbb{P}(m^\ba_i=k)=(1-\delta/2)^k\delta/2$, for $i = 1, \ldots, m^\ga$. Define $\tilde{m}$ to be the number of episodes by following policy $\pi$ over the horizon of $T$ time steps, $\tilde{m}^\ga$ to be the total number of selections of a good arm such that $\tilde{m}^\ga=\tilde{m}$ or $\tilde{m}^\ga=\tilde{m}-1$, and $\tilde{m}_i^\ba$ to be the number of selections of a bad arm in the $i$-th episode by the policy $\pi_1(\beta^\dagger)$ over the horizon $H$. Without loss of generality, we assume that the policy selects arms in the order of the sequence $\bar{a}_1,\bar{a}_2,\ldots.$ 

Under policy $\pi_1(\beta^\dagger)$, let $R_i^\ga$ be the regret (summation of mean reward gaps) contributed by pulling the good arm in the $i$-th episode and $R_{i,j}^\ba$ be the regret contributed by pulling the $j$-th bad arm in the $i$-th episode. Then let $R^{\pi_1(\beta^\dagger)}_{m^\ga}=\sum_{i=1}^{m^\ga}(R_i^\ga+\sum_{j\in[m_i^\ba]}R_{i,j}^\ba)$\footnote{Note that $R_{m^\ba}^{\pi_1(\beta^\dagger)}$ does not contain undefined $R_{i,j}^\ba$ such that $R_{i,j}^\ba$ when $m_i^\ba=0$.}, which is the regret over the period of $m^\ga$ episodes. 

 For $i\in [\tilde{m}^\ga]$, $j\in [\tilde{m}_i^\ba]$, let $a(i)$ be the sampled arm for the $i$-th good arm and $a(i,j)$ be a sampled arm for $j$-th bad arm in the $i$-th episode. Then $n_H(a(i))$ is the number of pulls of the good arm in the $i$-th episode and $n_H(a(i,j))$ is the number of pulls of the $j$-th bad arm in the $i$-th episode by the policy $\pi_1(\beta^\dagger)$ over the horizon $H$. Let $\tilde{a}$ be the last sampled arm over time horizon $H$ by $\pi_1(\beta^\dagger)$. We denote by $\hat{m}^\ga$ and $\hat{m}^\ba_i$ for $i\in[\tilde{m}]$ the number of arms excluding $\tilde{a}$ in the sampled $\tilde{m}^\ga$ number of good arms and $\tilde{m}^\ba_i$ number of bad arms for $i\in[\tilde{m}]$ as follows:  
 \begin{equation*}
    \hat{m}^\ga=
    \begin{cases}
    \tilde{m}^\ga-1 &\text{if  $\tilde{a}$ is a good arm}  \\
     \tilde{m}^\ga& \text{otherwise}
    \end{cases},
\end{equation*}
\begin{equation*}
        \hat{m}_i^\ba=\tilde{m}_i^\ba \text{ for $i\in[\tilde{m}-1]$},
    \text{ and }
    \hat{m}_{\tilde{m}}^\ba=
    \begin{cases}
    \tilde{m}_{\tilde{m}}^\ba  &
    \text{if $\tilde{a}$ is a good arm}  \\
     \tilde{m}_{\tilde{m}}^\ba-1 & \text{otherwise}.
    \end{cases}
\end{equation*}
These notations are defined only if they exist.\footnote{$n_H(a(i))$,  $n_H(a(i,j))$, and $\hat{m}_i^\ba$ are not defined for $i\in[0]$ or $j\in[0]$.}
 Excluding the last arm $\tilde{a}$ which the policy $\pi_1(\beta^\dagger)$ may stop to pull suddenly by reaching the horizon $H$, we provide lower bounds for the number of pulls for each arm, $n_H(a(i))$ for $i\in[\hat{m}^\ga]$ in the following lemma if they exist.

 \begin{lemma}
  Under $E_1$, given  $\hat{m}^\ga$, for any $i\in[\hat{m}^\ga]$ we have 
  $$
  n_H(a(i))\ge\delta/(2\beta^\dagger).
  $$
  \label{lem:n_low_bd_no_e}
 \end{lemma}
 \begin{proof}
 Let $$\hat{\mu}_t^o(a)=\frac{\sum_{s=1}^{t-1}(r_s+\rot_sn_{s}(a))\mathbbm{1}(a_s=a)}{n_{t}(a)},$$ which satisfies $\tilde{\mu}_t^o(a,\beta^\dagger)\ge \hat{\mu}_t^o(a)$ from $\beta^\dagger\ge\rot$.  Under $E_1$, it is true that for all $t\in[H]$ and $a\in\mathcal{A}(1)$, 
 $$|\hat{\mu}^o_t(a)-\mu_1(a)|\le \sqrt{10\log(H)/n_{t}(a)}.$$
 Let $a(i)$ be a sampled good arm in the $i$-th episode. Suppose that $n_{t}(a(i))=\lfloor\delta/(2\beta^\dagger)\rfloor$ for some $t>0$, then we have \begin{align*}
    \tilde{\mu}_t^o(a(i),\beta^{\dagger})-\beta^\dagger n_{t}(a(i))+\sqrt{10\log(
    H)/n_H(a(i))}&\ge\hat{\mu}_t^o(a(i))-\beta^\dagger n_{t}(a(i))+\sqrt{10\log(H)/n_H(a(i))}\cr &\ge \mu_1(a(i))-\delta/2\cr &\ge 1-\delta,
\end{align*}
where the second inequality is obtained from $E_1$ and  $n_{t}(a(i))\le \delta/(2\beta^\dagger)$, and the third inequality is from $\mu_1(a(i))\ge 1- \delta/2$. Therefore, policy $\pi_1(\beta^\dagger)$ must pull arm $a$ more times than $\lfloor\delta/(2\beta^\dagger)\rfloor$, which implies $n_H(a(i))\ge\delta/(2\beta^\dagger)$  for any $i\in[\hat{m}^\ga]$. 
 \end{proof}
\textbf{We first consider the case when $\rot=\omega(1/H^{3/2})$.} Then we have that $\beta^\dagger$ is the smallest value in $\mathcal{B}$ which exceeds $\rot$ such that $\rot\le \beta^\dagger\le 2\rot$.
For getting $R^{\pi_1(\beta^\dagger)}_{m^\ga}$, here we define how the policy $\pi_1(\beta^\dagger)$ works after time $H$ such that it pulls $\lceil\delta/(2\beta^\dagger)\rceil$ times a good arm and $0$ time a bad arm. We note that defining how $\pi_1(\beta^\dagger)$ works after $H$ is only for the proof to get a regret bound over time horizon $H$. For the last arm $\tilde{a}$ over the horizon $H$, it pulls the arm  up to $\max\{\lceil\delta/(2\beta^\dagger)\rceil,n_H(\tilde{a})\}$ times if $\tilde{a}$ is a good arm. We define that $n_H(a(i))=0$ for $i\in[m^\ga]/[\tilde{m}^\ga]$ for convenience.
For $i\in[m^\ga]$ and $j\in[m_i^\ba]$, let $n_i^\ga$ and $n_{i,j}^\ba$ be the number of pulls of the good arm in the $i$-th episode and the $j$-th bad arm in the $i$-th episode by the policy, respectively. Here we define $n_i^\ga$'s and $n_{i,j}^\ba$'s as follows:

If $\tilde{a}$ is a good arm, then
\begin{equation*}
    n_i^\ga=
    \begin{cases}
    n_H(a(i)) &\text{for } i\in[\tilde{m}^\ga-1]  \\
     \max\{\lceil\delta/(2\beta^\dagger)\rceil,n_H(a(i))\}& \text{for } i\in[m^\ga]/[\tilde{m}^\ga-1]
    \end{cases}, 
    \end{equation*}
\begin{equation*}
        n_{i,j}^\ba=
    \begin{cases}
    n_H(a(i,j))&\text{for } i\in[\tilde{m}^\ga],j\in[\tilde{m}_i^\ba]\\
    0 &\text{for } i\in[m^\ga]/[\tilde{m}^\ga],j\in[m^\ba_i]/[\tilde{m}_i^\ba].
    \end{cases}
\end{equation*}

Otherwise,
\begin{equation*}
    n_i^\ga=
    \begin{cases}
    n_H(a(i)) &\text{for } i\in[\tilde{m}^\ga]  \\
     \lceil\delta/(2\beta^\dagger)\rceil& \text{for } i\in[m^\ga]/[\tilde{m}^\ga]
    \end{cases}, 
    n_{i,j}^\ba=
    \begin{cases}
    n_H(a(i,j))&\text{for } i\in[\tilde{m}^\ga],j\in[\tilde{m}_i^\ba]\\
    0 &\text{for } i\in[m^\ga]/[\tilde{m}^\ga-1],j\in[m^\ba_i]/[\tilde{m}_i^\ba].
    \end{cases}
\end{equation*}
We note that $n_{i}^\ga$'s and $n_{i,j}^\ba$'s are defined only if they exist.\footnote{$n_i^\ga$ and $n_{i,j}^\ba$ are not defined for $i\in[0]$ or $j\in[0]$.}

We provide $m^{\ga}$ such that $R^{\pi_1(\beta^\dagger)}(H)\le R^{\pi_1(\beta^\dagger)}_{m^\ga}$ in the following lemma.

\begin{lemma}
Under $E_1$, when $m^\ga=\lceil 2H{\beta^\dagger}^{2/3}\rceil$, we have 
$$
R^{\pi_1(\beta^\dagger)}(H)\le R^{\pi_1(\beta^\dagger)}_{m^\ga}.
$$
\label{lem:regret_bd_prob_no_e}
\end{lemma}

\begin{proof}
 From Lemma~\ref{lem:n_low_bd_no_e}, with $\delta={\beta^\dagger}^{1/3}$ we have
 $$\sum_{i\in[m^\ga]}\left(n_i^\ga+\sum_{j\in[m_i^\ba]}n_{i,j}^\ba\right)\ge m^\ga \frac{\delta}{2\beta^\dagger}\ge H,$$
which implies that $R^{\pi_1(\beta^\dagger)}(H)\le R^{\pi_1(\beta^\dagger)}_{m^\ga}.$
\end{proof}
From the result of Lemma~\ref{lem:regret_bd_prob_no_e}, we set $m^\ga=\lceil 2H{\beta^\dagger}^{2/3}\rceil$.
  For getting a bound for $\mathbb{E}[R^{\pi_1(\beta^\dagger)}_{m^\ga}]$, we provide bounds for $\mathbb{E}[R^\ga_i]$ and $\mathbb{E}[R^\ba_{i,j}]$ in the following lemma.
  \begin{lemma} Under $E_1$ and policy $\pi_1(\beta^\dagger)$,
   for any $i\in[m^\ga]$ and $j\in[m^\ba_i]$,  we have
\begin{align*}
    \mathbb{E}[R_i^\ga]
    =\tilde{O}\left(\frac{\delta}{\rot^{2/3}}+\frac{\delta^2}{\rot}+\frac{1}{\rot^{1/3}}\right),
\end{align*}
and 
\begin{align*}
    \mathbb{E}[R_{i,j}^\ba]= \tilde{O}\left(1+\frac{\delta}{\rot^{1/3}}+\frac{\delta^2}{\rot^{2/3}}\right).
\end{align*}\label{lem:R_good_bad_bd_no_e}
  \end{lemma}
  \begin{proof}
First we provide a bound for $R^\ga_i$ using an upper bound of $n^\ga_i$. Recall that $a(i)$ is the sampled arm for the $i$-th good arm. We have 
\begin{align*}
     \tilde{\mu}_t^o(a(i),\beta^\dagger)-\beta^\dagger n_t(a(i))&=\hat{\mu}_t(a(i))+\beta^\dagger\frac{\sum_{s=1}^{t-1}n_s(a(i))\mathbbm{1}(a_s=a(i))}{n_t(a(i))}-\beta^\dagger n_t(a(i))\cr &
     \le \hat{\mu}_t(a(i))+\beta^\dagger \frac{n_t(a(i))+1}{2}-\beta^\dagger n_t(a(i))\cr&
     =\hat{\mu}_t(a(i))-\beta^\dagger \frac{n_t(a(i))-1}{2}.
   \end{align*}
   Since, under $E_1$,
   \begin{align*}
       \hat{\mu}_t(a(i))-(\beta^{\dagger}/2)(n_t(a(i))-1)&\le \bar{\mu}_t(a(i))+\sqrt{10\log(H)/n_t(a(i))}-(\beta^{\dagger}/2)(n_t(a(i))-1)\cr &\le  1+\sqrt{10\log(H)/n_t(a(i))}-(\beta^{\dagger}/2)(n_t(a(i))-1),
   \end{align*}
   for $i\in[\tilde{m}^\ga]$, from the policy $\pi$, we need to get $n$ such that
\begin{align}
    1-\frac{\beta^{\dagger}}{2}(n-1) +2\sqrt{10\log(H)/n}< 1-\delta,\label{eq:n_good_con_up_no_e}
\end{align} in which $n+1$ is an upper bound for $n^\ga_i$. Let $n_1=2(\delta+{\beta^{\dagger}}^{1/3})/\beta^{\dagger}+1$ and $n_2=C\log(H)/{\beta^{\dagger}}^{2/3}$ with some large enough constant $C>0$. Then $n=n_1+n_2$ satisfies \eqref{eq:n_good_con_up_no_e} because $1-\beta^{\dagger} (n_1-1)/2+2\sqrt{10\log(H)/n_2}< 1-\delta.$ Therefore, for all $i\in[\tilde{m}^\ga]$ we have 
$n_i^\ga=\tilde{O}((\delta+{\beta^{\dagger}}^{1/3})/\beta^{\dagger}).$
Then with the fact that $n_i^\ga=\lceil\delta/(2\beta^{\dagger})\rceil$ for $i\in[m^\ga]/[\tilde{m}^\ga]$ if they exist and $\beta^{\dagger}=\Theta(\rot)$, for any $i\in[m^\ga]$ we have $$n_i^\ga=\tilde{O}((\delta+{\beta^{\dagger}}^{1/3})/\beta^{\dagger})=\tilde{O}((\delta+\rot^{1/3})/\rot).$$  Then for any $i\in[m^\ga]$ we have 
\begin{align*}
    \mathbb{E}[R_i^\ga]&\le\mathbb{E}
    \left[\Delta_1(a(i))n_i^\ga+\frac{n_i^\ga(n_i^\ga-1)}{2}\rot\right]\cr
    &=\tilde{O}\left(\frac{2}{\delta}\int_0^{\delta/2}\frac{\delta+\rot^{1/3}}{\rot}x+\left(\frac{\delta+\rot^{1/3}}{\rot}\right)^2\rot dx\right)\cr
    &=\tilde{O}\left(\frac{\delta}{\rot^{2/3}}+\frac{\delta^2}{\rot}+\frac{1}{\rot^{1/3}}\right),
\end{align*}
where the first equality is obtained from the fact that $\Delta_1(a(i))$'s are i.i.d. random variables with uniform distribution on $[0,\delta/2]$ and $n_i^\ga=\tilde{O}((\delta+\rot^{1/3})/\rot)$.

Now we provide an upper bound of $n^\ba_{i,j}$ to get a bound of $R_{i,j}^\ba$ for $i\in[m^\ga],j\in[m_i^\ba]$. Let $a(i,j)$ be a sampled arm for $j$-th bad arm in the $i$-th episode. When $\delta/2<\Delta_1(a(i,j))\le \delta+\rot^{1/3}$, as in the case of the good arm, $n=2(\delta+{\beta^{\dagger}}^{1/3})/\beta^{\dagger} +1+C_1\log(H)/{\beta^{\dagger}}^{2/3}$ for some large enough constant $C_1>0$, satisfies $\eqref{eq:n_good_con_up_no_e}$ so that $n_{i,j}^\ba=\tilde{O}((\delta+\rot^{1/3})/\rot)$ for all $i\in[\tilde{m}], j\in[\tilde{m}_i^\ba]$. When $\delta+\rot^{1/3}<\Delta_1(a(i,j))\le 1$, since under $E_1$
\begin{align*}
    \tilde{\mu}_t^o(a(i,j))-\beta^{\dagger} n_t(a(i,j))&=\hat{\mu}_t(a(i,j))+\beta^{\dagger}\frac{\sum_{s=1}^{t-1}n_s(a(i,j))\mathbbm{1}(a_s=a(i,j))}{n_t(a(i,j))}-\beta^{\dagger} n_t(a(i,j))\cr &
     \le \hat{\mu}_t(a(i,j))+\beta^{\dagger} \frac{n_t(a(i,j))+1}{2}-\beta^{\dagger} n_t(a(i,j))\cr&
     =\hat{\mu}_t(a(i,j))-\beta^{\dagger} \frac{n_t(a(i,j))-1}{2}\cr &\le \bar{\mu}_t(a(i,j))+\sqrt{10\log(H)/n_t(a(i,j))}-(\beta^{\dagger}/2)(n_t(a(i,j))-1)\cr &\le  \mu_1(a(i,j))+\sqrt{10\log(H)/n_t(a(i,j))}-(\beta^{\dagger}/2)(n_t(a(i,j))-1),
\end{align*}
from the policy $\pi$ we need to get $n\ge 1$ such that
\begin{align}
    \mu_1(a(i,j))-\frac{\beta^\dagger}{2}(n-1)+2\sqrt{10\log(H)/n}< 1-\delta, \label{eq:n_bad_con_up_no_e}
\end{align}
in which $n+1$ is an upper bound of $n_{i,j}^\ba$ for $i\in[\tilde{m}], j\in[\tilde{m}_i^\ba]$. From a sufficient condition for \eqref{eq:n_bad_con_up_no_e} to hold such that $$\mu_1(a(i,j))+2\sqrt{10\log(H)/n}< 1-\delta,$$ we can find that $n=C_2\log(H)/(\Delta_1(a(i,j))-\delta)^2$
for some large constant $C_2>0$ satisfies \eqref{eq:n_bad_con_up_no_e}. Therefore, when $\delta+\rot^{1/3}<\Delta_1(a(i,j))\le 1$, for all $i\in[\tilde{m}],j\in[\tilde{m}_i^\ba]$ we have $n_{i,j}^\ba=\tilde{O}(1/(\Delta_1(a(i,j))-\delta)^2).$ Then with the fact that $n_{i,j}^\ba=0$ for $i\in[m^\ga]/[\tilde{m}^\ga]$, $j\in[m_i^\ba]/[\tilde{m}_i^\ba]$ if they exist, for any $i\in[m^\ga]$ and $j\in[m^\ba_i]$, we have
\begin{equation*}
    n_{i,j}^\ba=
    \begin{cases}
    \tilde{O}((\delta+\rot^{1/3})/\rot) & \text{if $\delta/2<\Delta_1(a(i,j))\le\delta+\rot^{1/3}$}\\
    \tilde{O}(1/(\Delta_1(a(i,j))-\delta)^2) & \text{if $\delta+\rot^{1/3}<\Delta_1(a(i,j))\le 1$}
    \end{cases}
\end{equation*}
Then for any $i\in[m^\ga]$, $j\in[m^\ba_i]$, we have 
\begin{align*}
    \mathbb{E}[R_{i,j}^\ba]&\le\mathbb{E}
    \left[\Delta_1(a(i,j))n_{i,j}^\ba+\frac{n_{i,j}^\ba(n_{i,j}^\ba-1)}{2}\rot\right]\cr
    &=\tilde{O}\left(\frac{1}{1-\delta/2}\left(\int_{\delta/2}^{\delta+\rot^{1/3}}\frac{\delta+\rot^{1/3}}{\rot}x+\left(\frac{\delta+\rot^{1/3}}{\rot}\right)^2\rot dx\right . \right . + \left . \left .\int_{\delta+\rot^{1/3}}^1 \frac{1}{(x-\delta)^2}x+\frac{1}{(x-\delta)^4}\rot dx \right)\right)\cr
    &= \tilde{O}\left(1+\frac{\delta}{\rot^{1/3}}+\frac{\delta^2}{\rot^{2/3}}\right).
\end{align*}
  \end{proof}
 
Recall that $R^{\pi_1(\beta^\dagger)}_{m^\ga}=\sum_{i=1}^{m^\ga}(R_i^\ga+\sum_{j\in[m_i^\ba]}R_{i,j}^\ba)$. With $\beta^\dagger=\Theta(\rot),$ $\delta=\Theta(\rot^{1/3})$, and $m^\ga=\lceil 2H{\beta^\dagger}^{2/3}\rceil$, from Lemmas~\ref{lem:regret_bd_prob_no_e} and \ref{lem:R_good_bad_bd_no_e}, and the fact that $m_i^\ba$'s are i.i.d. random variables with geometric distribution with $\mathbb{E}[m_i^\ba]=2/\delta-1$, we have
\begin{align}
\mathbb{E}[R^{\pi_1(\beta^\dagger)}(H)]&=O(\mathbb{E}[R^{\pi_1(\beta^\dagger)}_{m^\ga}])\cr &=O\left(\mathbb{E}\left[\sum_{i=1}^{m^\ga}\left(R^\ga_i+\sum_{j\in[m^\ba_i]}R^\ba_{i,j}\right)\right]\right)\cr
  &= \tilde{O}\left(H\rot^{2/3}\left(\left(\frac{\delta}{\rot^{2/3}}+\frac{\delta^2}{\rot}+\frac{1}{\rot^{1/3}}\right)+\frac{1}{\delta}\left(1+\frac{\delta}{\rot^{1/3}}\right)\right)\right)\cr
  &=\tilde{O}\left(H\rot^{2/3}\left(\frac{\delta}{\rot^{2/3}}+\frac{\delta^2}{\rot}+\frac{1}{\delta}+\frac{1}{\rot^{1/3}}\right)\right)\cr
  &=\tilde{O}\left(H\rot^{1/3}\right).\label{eq:regret_bd_large_no_e}
\end{align}

\textbf{Now we consider the case when $\rot=O(1/H^{3/2}).$} We have set $\beta^\dagger$ as the smallest element in $\mathcal{B}$ that exceeds  $\max\{\rot,1/H^{3/2}\}$; hence we have $ \beta^\dagger=\Theta(1/H^{3/2})$. 
For getting $R^{\pi_1(\beta^\dagger)}_{m^\ga}$, here we define how the policy $\pi_1(\beta^\dagger)$ works after $H$ time steps such that it pulls $H$ times a good arm and zero a bad arm. For the last arm $\tilde{a}$ over the horizon $H$, it pulls the arm up to $H$ times if $\tilde{a}$ is a good arm. With slight abuse of notation,
for $i\in[m^\ga]$ and $j\in[m_i^\ba]$, let $n_i^\ga$ and $n_{i,j}^\ba$ be the number of pulls of the good arm in the $i$-th episode and the $j$-th bad arm in the $i$-th episode by the policy, respectively. Here we define $n_i^\ga$'s and $n_{i,j}^\ba$'s as follows: 

If $\tilde{a}$ is a good arm, then
\begin{equation*}
    n_i^\ga=
    \begin{cases}
    n_H(a(i)) &\text{for } i\in[\tilde{m}^\ga-1]  \\
     H & \text{for } i\in[m^\ga]/[\tilde{m}^\ga-1]
    \end{cases}, 
    n_{i,j}^\ba=
    \begin{cases}
    n_H(a(i,j))&\text{for } i\in[\tilde{m}^\ga],j\in[\tilde{m}_i^\ba]\\
    0 &\text{for } i\in[m^\ga]/[\tilde{m}^\ga],j\in[m^\ba_i]/[\tilde{m}_i^\ba].
    \end{cases}
\end{equation*}
Otherwise,
\begin{equation*}
    n_i^\ga=
    \begin{cases}
    n_H(a(i)) &\text{for } i\in[\tilde{m}^\ga]  \\
     H & \text{for } i\in[m^\ga]/[\tilde{m}^\ga]
    \end{cases}, 
    n_{i,j}^\ba=
    \begin{cases}
    n_H(a(i,j))&\text{for } i\in[\tilde{m}^\ga],j\in[\tilde{m}_i^\ba]\\
    0 &\text{for } i\in[m^\ga]/[\tilde{m}^\ga-1],j\in[m^\ba_i]/[\tilde{m}_i^\ba].
    \end{cases}
\end{equation*}


From Lemma~\ref{lem:n_low_bd_no_e}, we observe that $n_i^\ga\ge \min\{\delta/(2\beta^{\dagger}),H\}$ for $i\in[m^\ga]$.
Then under $E_1$, if $m^\ga=C_3$ for some large enough constant $C_3>0$, then with $\delta={\beta^\dagger}^{1/3}$ and $\beta^\dagger=\Theta( 1/H^{3/2})$, we have 
$$
\sum_{i\in[m^\ga]}n_i^\ga\ge H,
$$ 
which implies $R^{\pi_1(\beta^\dagger)}(H)\le R^{\pi_1(\beta^\dagger)}_{m^\ga}$.
Therefore, we set $m^\ga=C_3$.
  For getting a bound for $\mathbb{E}[R^{\pi_1(\beta^\dagger)}_{m^\ga}]$, we provide bounds for $\mathbb{E}[R^\ga_i]$ and $\mathbb{E}[R^\ba_{i,j}]$ in the following lemma.

\begin{lemma}
Under $E_1$ and policy $\pi_1(\beta^\dagger)$, for any $i\in[m^\ga]$ and $j\in[m^\ba_i]$, we have
\begin{align*}
    \mathbb{E}[R_i^\ga]
    =O\left(\delta H+H^2\rot\right),
\end{align*}
and 
\begin{align*}
    \mathbb{E}[R_{i,j}^\ba]=\tilde{O}\left(H\delta^2+\sqrt{H}\delta+1\right).
\end{align*}\label{lem:R_good_bad_bd_small_no_e}
  \end{lemma}
  \begin{proof}
   First we provide a bound for $R^\ga_i$ using an upper bound of $n^\ga_i$. With the definition of $n_i^\ga=H$ for $i\in[m^\ga]/[\tilde{m}^\ga]$, for any $i\in[m^\ga]$ we have $$n_i^\ga\le H.$$ Recall that $a(i)$ is the sampled arm for the $i$-th good arm. Then, for any $i\in[m^\ga]$, we have 
\begin{align*}
    \mathbb{E}[R_i^\ga]&\le\mathbb{E}
    \left[\Delta_1(a(i))n_i^\ga+\frac{n_i^\ga(n_i^\ga-1)}{2}\rot\right]\cr
    &=O\left(\delta H+H^2\rot \right),
\end{align*}
where the first equality is obtained from the fact that $\Delta_1(a(i))$'s are i.i.d. random variables with uniform distribution on $[0,\delta/2]$ and $n_i^\ga\le H$.

Now we provide an upper bound of $n^\ba_{i,j}$ to get a bound of $R_{i,j}^\ba$ for $i\in[m^\ga],j\in[m_i^\ba]$. Let $a(i,j)$ be a sampled arm for $j$-th bad arm in the $i$-th episode. When $\delta/2<\Delta_1(a(i,j))\le \delta+1/\sqrt{H}$, as in the case of the good arm, $n_{i,j}^\ba\le H$ for all $i\in[\tilde{m}], j\in[\tilde{m}_i^\ba]$. When $\delta+1/\sqrt{H}<\Delta_1(a(i,j))\le 1$, from the policy $\pi_1(\beta^\dagger)$ under $E_1$,  we need to get $n\ge 1$ such that
\begin{align}
    \mu_1(a(i,j))-\frac{\beta^\dagger}{2} (n-1)+2\sqrt{10\log(H)/n}< 1-\delta, \label{eq:n_bad_con_up_small_no_e}
\end{align}
in which $n+1$ is an upper bound of $n_{i,j}^\ba$ for $i\in[\tilde{m}], j\in[\tilde{m}_i^\ba]$. From a sufficient condition for \eqref{eq:n_bad_con_up_small_no_e} to hold such that $$\mu_1(a(i,j))+2\sqrt{10\log(H)/n}< 1-\delta,$$ we can find that $n=C_4\log(H)/(\Delta_1(a(i,j))-\delta)^2$
for some large constant $C_4>0$ satisfies \eqref{eq:n_bad_con_up_small_no_e}. Therefore, when $\delta+1/\sqrt{H}<\Delta_1(a(i,j))\le 1$, for all $i\in[\tilde{m}],j\in[\tilde{m}_i^\ba]$ we have $n_{i,j}^\ba=\tilde{O}(1/(\Delta_1(a(i,j))-\delta)^2).$ Then with the fact that $n_{i,j}^\ba=0$ for $i\in[m^\ga]/[\tilde{m}^\ga]$, $j\in[m_i^\ba]/[\tilde{m}_i^\ba]$, for any $i\in[m^\ga]$, $j\in[m^\ba_i]$, if $\delta/2<\Delta_1(a(i,j))\le\delta+1/\sqrt{H}$, we have 
$$
n_{i,j}^\ba\le H,
$$ 
and if $\delta+1/\sqrt{H}<\Delta_1(a(i,j))\le 1$, we have 
$$
n_{i,j}^\ba=\tilde{O}(1/(\Delta_1(a(i,j))-\delta)^2).
$$ 
For any $i\in[m^\ga]$, $j\in[m^\ba_i]$, we obtain 
\begin{align*}
    \mathbb{E}[R_{i,j}^\ba]&\le\mathbb{E}
    \left[\Delta_1(a(i,j))n_{i,j}^\ba+\frac{n_{i,j}^\ba(n_{i,j}^\ba-1)}{2}\rot\right]\cr
    &=\tilde{O}\left(\frac{1}{1-\delta/2}\left(\int_{\delta/2}^{\delta+1/\sqrt{H}}Hx+H^2\rot dx\right . \right . + \left . \left .\int_{\delta+1/\sqrt{H}}^1 \frac{1}{(x-\delta)^2}x+\frac{1}{(x-\delta)^4}\rot dx \right)\right)\cr
    &= \tilde{O}\left(H\delta^2+\sqrt{H}\delta+1\right).
\end{align*}
  \end{proof}

It follows that, with $\delta=\Theta(1/\sqrt{H})$ and $m^\ga=C_3$, we have
\begin{align}
\mathbb{E}[R^{\pi_1(\beta^\dagger)}(H)]&=O(\mathbb{E}[R^{\pi_1(\beta^\dagger)}_{m^\ga}])\cr
&=O\left(\mathbb{E}\left[\sum_{i\in[m^\ga]}\left(R_i^\ga+\sum_{j\in[m_i^\ba]}R_{i,j}^\ba\right)\right)\right]\cr
&=\tilde{O}\left(\left(\delta H+H^2\rot\right)+\frac{1}{\delta}\left(H\delta^2+\sqrt{H}\delta+1\right)\right)\cr
&=\tilde{O}(\sqrt{H}),    \label{eq:regret_bd_small_no_e}
\end{align}
where the third equality is obtained from Lemma~\ref{lem:R_good_bad_bd_small_no_e} and  $\mathbb{E}[m_i^\ba]=2/\delta-1$.

Finally, we can conclude the proof by noting that from \eqref{eq:regret_bd_large_no_e} and \eqref{eq:regret_bd_small_no_e}, for $\rot=o(1)$, we have
$$\mathbb{E}[R^{\pi_1(\beta^\dagger)}(H)]=\tilde{O}(\max\{H\rot^{1/3},\sqrt{H}\}).$$ 
Therefore, by summing regrets from $\lceil T/H\rceil$ number of blocks, we have shown that
\begin{align}
\mathbb{E}[R_1^\pi(T)]=\tilde{O}((T/H)\max\{H\rot^{1/3},\sqrt{H}\})=\tilde{O}(\max\{T\rot^{1/3},T/\sqrt{H}\}).\label{eq:regret_bd_aducb_e}
\end{align}
 

\paragraph{Upper bounding $\mathbb{E}[R_2^\pi(T)]$} We observe that the EXP3 is run for $\lceil T/H \rceil$ decision rounds and the number of policies (i.e. $\pi_i(\beta)$ for $\beta\in\mathcal{B}$) is $B$. Denote the maximum absolute sum of rewards of any block with length $H$ by a random variable $Q^\prime$. 
We first provide a bound for $Q^\prime$ using concentration inequalities. For any block $i$, we have 
\begin{align}
    \left|\sum_{t=(i-1)H+1}^{i\cdot H\wedge T}\mu_t(a_t)+\eta_t\right|\le \left|\sum_{t=(i-1)H+1}^{i\cdot H\wedge T}\mu_t(a_t)\right|+\left|\sum_{t=(i-1)H+1}^{i\cdot H\wedge T}\eta_t\right|.\label{eq:Q_bd_no_e}
\end{align}
Denote by $\mathcal{T}_i$ the set of time steps in the $i$-th block. We define the event $E_2(i)=\{|\hat{\mu}_t(a)-\bar{\mu}_t(a)|\le \sqrt{8\log(T)/n_t(a)}, \hbox{ for all } t\in \mathcal{T}_i, a\in\mathcal{A}(i)\}$ and $E_2=\bigcap_{i\in[\lceil T/H \rceil]}E_2(i).$ From Lemma~\ref{lem:chernoff_sub-gau}, with $H=\sqrt{T}$ we have 
$$
\mathbb{P}(E_2^c)\le \sum_{i\in[\lceil T/H\rceil]}\frac{2H^3}{T^{4}}\le \frac{2}{T^2}.
$$ 
By assuming that $E_2$ holds true, we can get a lower bound for $\mu_t(a_t)$, which may be a negative value from rotting, 
for getting an upper bound for $|\sum_{t=(i-1)H+1}^{i\cdot H\wedge T}\mu_t(a_t)|$. Let $\beta_{\max}$ denote the maximum value in $\mathcal{B}$.  From the policy $\pi$ with $H=\sqrt{T}$, when $\mu_1(a)=0$ for some arm $a$, since 
\begin{align*}
    &\tilde{\mu}_t^o(a(i,j))-\rot n_t(a(i,j))+\sqrt{\frac{10\log(H)}{n_t(a(i,j))}}\cr &\le \hat{\mu}_t(a(i,j))-(\rot/2)(n_t(a(i,j))-1)+\sqrt{\frac{4\log(T)}{n_t(a(i,j))}}\cr
    &= \bar{\mu}_t(a(i,j))-(\rot/2)(n_t(a(i,j))-1)+\sqrt{\frac{4\log(T)}{n_t(a(i,j))}}+\sqrt{
    \frac{8\log(T)}{n_t(a(i,j))}}
    \cr &\le\mu_1(a)-(\rot/2)(n_t(a(i,j))-1)+\sqrt{\frac{4\log(T)}{n_t(a(i,j))}}+\sqrt{
    \frac{8\log(T)}{n_t(a(i,j))}}\cr &\le \sqrt{\frac{4\log(T)}{n_t(a(i,j))}}+\sqrt{
    \frac{8\log(T)}{n_t(a(i,j))}}, 
\end{align*}
we need to find an positive integer $n$ such that
   \begin{align*}
       \sqrt{4\log(T)/n}+\sqrt{8\log(T)/n}\le 1-\beta_{\max}^{1/3},
   \end{align*} 
   in which $n$ is an upper bound for the number of pulls of arm $a$. From $\beta_{\max}^{1/3}=1/2$, we can observe that the condition holds with $n=\lceil C\log(T)\rceil$ for some large enough $C$. From this fact, we can find that for any sampled arm $a$ from $\pi$, we have $\mu_t(a)\ge -(C\log(T)+1)\rot  $. Then under $E_2$, with $\mu_t(a)\le 1$, for any $i\in[\lceil T/H\rceil]$ we have $|\sum_{t=(i-1)H+1}^{i\cdot H\wedge T}\mu_t(a_t)|\le\max\{ (C \log(T)+1)\rot H,H\}$. 

Next we provide a bound for $|\sum_{t=(i-1)H+1}^{i\cdot H\wedge T}\eta_t|$. We define the event $E_3(i)=\{|\sum_{t=(i-1)H+1}^{i\cdot H\wedge T}\eta_t| \le 2\sqrt{H\log( T)}\}$ and $E_3=\bigcap_{i\in[\lceil T/H\rceil]}E_3(i)$. From Lemma~\ref{lem:chernoff_sub-gau}, for any $i\in[\lceil T/H\rceil]$, we have
$$
\mathbb{P}\left(E_3(i)^c\right)\le \frac{2}{T^2}.
$$
Then, under $E_2\cap E_3$, with \eqref{eq:Q_bd_no_e}, we have 
$$
Q^\prime\le \max\{ C H\log(T)\rot,H\}+2\sqrt{H\log(T)}\le 93 H\log(T)+2\sqrt{H\log(T)} ,
$$ which implies
 $1/2+\sum_{t=(i-1)H}^{i\cdot H\wedge T}r_t/(2CH\log T+4\sqrt{H\log T})\in[0,1]$. With the rescaling and translation of rewards in Algorithm~\ref{alg:Alg2}, from Corollary 3.2. in \cite{auer}, we have    
\begin{align}  
\mathbb{E}[R_2^\pi(T)|E_2\cap E_3]= \tilde{O}\left((CH\log T+2\sqrt{H\log T})\sqrt{BT/H}\right)=\tilde{O}\left(\sqrt{HBT}\right).\label{eq:regret_bd_exp3_Q}
\end{align}
Note that the expected regret from EXP3 is trivially bounded by $o(H^2(T/H))=o(TH)$ and $B=O(\log(T))$. Then, with \eqref{eq:regret_bd_exp3_Q}, we have
\begin{align}
\mathbb{E}[R_2^\pi(T)]
&=\mathbb{E}[R_2^\pi(T)|E_2\cap E_3]\mathbb{P}(E_2 \cap E_3)+\mathbb{E}[R_2^\pi(T)|E_2^c\cup E_3^c]\mathbb{P}(E_2^c \cup E_3^c)\cr
    &= \tilde{O}\left(\sqrt{HT}\right)+o\left(TH\right)(4/T^2)\cr
    &= \tilde{O}\left(\sqrt{HT}\right). \label{eq:regret_bd_exp3_e}
\end{align}
Finally, from \eqref{eq:regret_up_bd_bob}, \eqref{eq:regret_bd_aducb_e}, and \eqref{eq:regret_bd_exp3_e}, with $H=\sqrt{T}$, we have
$$
\mathbb{E}[R^\pi(T)]=\tilde{O}\left(\max\{T^{3/4},\rot^{1/3}T\}\right),$$ which concludes the proof.

\end{document}